\documentclass{article}

    \PassOptionsToPackage{numbers, compress}{natbib}

\usepackage[main, final]{neurips_2026}

\usepackage[utf8]{inputenc} 
\usepackage[T1]{fontenc}    
\usepackage[backref=page]{hyperref}       
\usepackage{microtype}
\usepackage{graphicx}
\usepackage{booktabs} 
\usepackage{makecell}
\usepackage{xcolor}         
\usepackage{wrapfig}
\usepackage{adjustbox}
\usepackage{import}

\definecolor{linkcolor}{RGB}{83,83,182}

\hypersetup
{
  colorlinks=true,
  citecolor=linkcolor,
  linkcolor= linkcolor,
  urlcolor=blue,
  pdfstartview=FitH,
  pdftitle={Partition Tree: Conditional Density Estimation over General Outcome Spaces},
  pdfauthor={Felipe Angelim and Alessandro Leite},
  pdfcreator={},
  pdfsubject={},
  pdfkeywords={},     
}

\usepackage{amsmath}
\usepackage{amssymb}
\usepackage{mathtools}
\usepackage{amsthm}
\usepackage{algorithm}
\usepackage{algorithmic}
\usepackage{thmtools}
\usepackage{thm-restate}
\usepackage{paralist}
\usepackage{subcaption}
\usepackage{multirow}

\usepackage{etoc}
\usepackage{titletoc}

\usepackage[capitalize,noabbrev]{cleveref}

\theoremstyle{plain}
\newtheorem{theorem}{Theorem}[section]
\newtheorem{proposition}[theorem]{Proposition}
\newtheorem{lemma}[theorem]{Lemma}
\newtheorem{corollary}[theorem]{Corollary}
\theoremstyle{definition}

\newtheorem{assumption}[theorem]{Assumption}
\newtheorem{claim}{Claim}
\theoremstyle{remark}

\usepackage[disable,textsize=tiny]{todonotes}

\newcommand{\X}{\mathcal X}
\newcommand{\Y}{\mathcal Y}
\newcommand{\D}{\mathcal D}
\newcommand{\Z}{\mathcal Z}
\newcommand{\bZ}{\bZN}
\newcommand{\bYN}{\bar{\Y}_N}
\newcommand{\bZN}{\bar{\Z}_N}
\newcommand{\A}{\mathcal A}

\newcommand{\PP}{\mathbb P}
\newcommand{\E}{\mathbb E}
\newcommand{\1}{\mathbf 1}
\newcommand{\YN}{\bYN}

\newcommand{\tendsto}{\xrightarrow[]{N\rightarrow\infty}}
\newcommand{\PX}{\PP_X}
\newcommand{\PY}{\PP_Y}

\newcommand{\hPXY}{\hat{\PP}_{XY}}
\newcommand{\PXY}{\PP_{XY}}
\newcommand{\hPX}{\hat{\PP}_X}

\newcommand{\Xc}{\mathcal X^{(c)}}   
\newcommand{\Xk}{\mathcal X^{(k)}}   
\newcommand{\Yc}{\mathcal Y^{(c)}}

\newcommand{\Yk}{\mathcal Y^{(k)}}

\newcommand{\dx}{d_x}

\newcommand{\dxc}{d_{x,c}}

\newcommand{\dxk}{d_{x,k}}
\newcommand{\dyc}{d_{y,c}}

\newcommand{\dyk}{d_{y,k}}

\newcommand{\dham}{d_{\mathrm{Ham}}}
\newcommand{\dcont}{d_{cont}}
\newcommand{\distZ}{\mathrm{dist}_{d_{\Z}}}

\newcommand{\diamh}{\mathrm{diam}}

\newcommand{\diamc}{\mathrm{diam}_{cont}}

\newcommand{\fn}{f_N}
\newcommand{\fxn}{f_N'}
\newcommand{\fs}{f^*_N}
\newcommand{\fstar}{f^*}

\newcommand{\proofparagraph}[1]{\medskip\noindent\textbf{#1.}}

\newcommand{\Lo}{\mathcal L}

\newcommand{\lnu}{L^1(\nu)}
\newcommand{\lnuz}{L^1(\nu;\bZN)}

\newcommand{\fpart}{\hat{f}_{\pi_N}}
\newcommand{\fpartn}{\bar{f}_{\pi_N}}

\newcommand{\pts}{\emph{Partition Trees}}
\newcommand{\pt}{Partition Tree}
\newcommand{\pfs}{\emph{Partition Forests}}
\newcommand{\pf}{Partition Forest}

\newcommand{\ie}{i.e.,}
\newcommand{\eg}{e.g.,}
\newcommand{\code}{\href{https://github.com/felipeangelimvieira/partition_tree}{github.com/felipeangelimvieira/partition\_tree}}

\graphicspath{{figures/}{figures/feature_importance/}}

\newcolumntype{H}{>{\setbox0=\hbox\bgroup}c<{\egroup}@{}}

\crefname{ALC@unique}{Line}{Lines}
\Crefname{ALC@unique}{Line}{Lines}
\crefname{algorithm}{Alg.}{Algs.}
\crefname{table}{Tab.}{Tabs.}
\crefname{figure}{Fig.}{Figs.}
\crefname{section}{Sec.}{Secs.}
\crefname{theorem}{Thm}{Thms}
\crefname{lemma}{Lem.}{Lems.}
\crefname{corollary}{Cor.}{Cors.}
\crefname{appendix}{App.}{Apps.}
\newcommand{\scikit}{\texttt{scikit-learn}}

\everymath{\displaystyle}

\title{Partition Tree: Conditional Density Estimation over General Outcome Spaces}

\author{%
  Felipe Angelim\thanks{Corresponding authors.} \\
  Independent Researcher \\
  \texttt{felipeangelim@pm.me}
  \And
  Alessandro Leite\footnotemark[1] \\
  INSA Rouen Normandy, Normandy University\\
  \texttt{aleite@insa-rouen.fr}
}

\begin{document}

\maketitle

\begin{abstract}
We propose Partition Tree, a tree-based framework for conditional density estimation over general outcome spaces that supports both continuous and categorical variables within a unified formulation. Our approach models conditional distributions as piecewise-constant densities on data-adaptive partitions and learns trees by greedily minimizing a surrogate for conditional log-loss. This yields a scalable alternative to existing probabilistic trees that does not make parametric assumptions about the target distribution. We further introduce Partition Forest, an ensemble extension obtained by averaging conditional densities. Empirically, we demonstrate improved probabilistic prediction over CART-style trees and competitive performance compared to state-of-the-art probabilistic tree methods and Random Forests.
\end{abstract}

\section{Introduction}

Decision trees are among the earliest and most widely used learning algorithms, valued for their computational efficiency, intrinsic interpretability, and ability to capture nonlinear interactions through recursive partitioning \cite{morgan1963problems,kass1980exploratory,breiman2017classification,loh2014fifty}. Early methods such as Automatic Interaction Detection (AID)~\cite{morgan1963problems} and CHAID \cite{kass1980exploratory} introduced the idea of fitting piecewise-constant models on data-adaptive partitions of the input space. The CART framework later systematized tree construction for both classification and regression and popularized pruning via cost-complexity regularization \cite{breiman2017classification}.

Although classical trees focus on point prediction or class probabilities, there has been growing interest in extending tree-based methods to probabilistic prediction and conditional density estimation. Recent approaches include CADET~\cite{cousins2019cadet}, which employs parametric conditional densities at the leaves and selects splits using a cross-entropy criterion, and CDTree~\cite{yang2024conditional}, which jointly optimizes splits and leaf-wise histograms via a minimal description length objective. Although these methods demonstrate strong empirical performance, CADET's leaf models impose parametric assumptions, whereas CDTree's joint split-and-histogram optimization can be computationally demanding, limiting scalability on large datasets.

In this paper, we introduce \emph{\pt}, a tree-based framework for conditional density estimation over general outcome spaces that supports both continuous and categorical variables. Our approach models conditional distributions as piecewise-constant densities defined on data-adaptive partitions of the joint covariate-outcome space. Formally, we adopt a measure-theoretic perspective in which conditional densities are defined via probability measures and Radon–Nikodym derivatives with respect to these partitions. Within this unified framework, both classification and regression emerge as cases of conditional density estimation. Trees are constructed by greedily minimizing an empirical surrogate for conditional log-loss, which upper-bounds the normalized training negative log-likelihood. This objective naturally accommodates heteroscedastic noise, and we evaluate its sensitivity to redundant features.

We further extend \pt\ to ensembles, termed \pfs, by averaging the predicted conditional densities across multiple trees. Experiments on a diverse set of classification and regression benchmarks demonstrate that \pfs\ consistently improve probabilistic prediction compared to CART-style trees. On classification tasks, they achieve lower log-loss across several datasets, while \pfs\ outperform Random Forests on most benchmarks. For regression, \pfs\ provide competitive point predictions and provide strong probabilistic performance, outperforming CADET on larger datasets without relying on parametric assumptions. Finally, semi-synthetic experiments analyze the robustness to both homoscedastic and heteroscedastic noise, as well as to redundant features. Code and source files are available at \code.

\section{Methodology}\label{sec:methodology}

\subsection{Setting}\label{sec:setting}

We first introduce the measure-theoretic notation used throughout the paper.
This formulation makes explicit the connection between the underlying
measure-theoretic objects and the tree operations used by our method, and it
enables a unified treatment of continuous and categorical variables.

Let $(\Omega,\mathcal A,\PP)$ be a probability space. Let $X\in\X$ denote the covariates and
$Y\in\Y$ denote the outcome. We observe a dataset
$\D=\{(x_i,y_i)\}_{i=1}^N$ of i.i.d.\ samples drawn from the joint distribution
$\PP_{XY}$, the push-forward of $\PP$ on the product space
$\Z:=\X\times\Y$.

A measurable cell $A\subseteq\Z$ is written as a product set
$A=A_X\times A_Y$, where $A_X\subseteq\X$ and $A_Y\subseteq\Y$ denote its covariate
and outcome components, respectively. We use subscripts to indicate projections of a cell onto a specific coordinate space, and omit subscripts when referring to
subsets of the joint space. Let $\mathrm{pr}_X:\X\times\Y\to\X$ and
$\mathrm{pr}_Y:\X\times\Y\to\Y$ denote the canonical coordinate projections.
Accordingly, for any $A\subseteq\X\times\Y$, we define
$A_X:=\mathrm{pr}_X(A)$ and $A_Y:=\mathrm{pr}_Y(A)$.
Let
\begin{align}
n_{XY}(A) &:= \sum\limits_{i=1}^{N} \1\{(x_i, y_i) \in A\} = N\;\hPXY(A) \label{eq:nxy}\\
n_X(A) &:= \sum\limits_{i=1}^{N} \1\{x_i \in A_X\} = N\;\hPX(A_X) \label{eq:nx}
\end{align}

denote the number of samples in a cell and in its $\X$-projection. We use $\PX(A) := \PX(A_X)$, $\PY(A) := \PY(A_Y)$, $\mu_Y(A) := \mu_Y(A_Y)$ for simplicity. In addition, let $\pi_N(\{z_i\}_{i=1}^N)$, $z_i \in \Z$, be a partitioning rule that maps datasets of size $N$ to measurable partitions $\pi_N(\{z_i\}_{i=1}^N) = \{ A_{i}\}_{i=1}^K$. We overload the notation $\pi_N$ to also denote the partition generated by $\pi_N(\{z_i\}_{i=1}^N)$ when the dataset is fixed or irrelevant from context. In particular, $\pi_N(\D)$ implies that $|\D|=N$. For any $z\in\Z$, we denote by $\pi_N[z]$ the unique cell $A\in\pi_N$ such that
$z\in A$.

We allow both $\X$ and $\Y$ to contain continuous and categorical coordinates.
Throughout, cells preserve a product structure: continuous coordinates are
restricted by intervals, categorical coordinates by subsets of their alphabets,
and candidate splits act on one coordinate at a time. For mixed-type outcomes, 
$\mu_Y$ is understood as the product of Lebesgue measure on continuous
coordinates and counting measure on categorical ones. The metric notions used
later in the consistency argument are deferred to the appendix.

%
\subsection{Conditional Density Estimation via Radon--Nikodym Derivatives}\label{sec:estimating_density}



Let $(X, Y)$ be a random pair taking values in $\X \times \Y$, where $\X$ and $\Y$ are measurable spaces. Let $\PXY$ and $\PX$ denote the joint and marginal distributions, respectively. We fix a $\sigma$-finite reference measure $\mu_Y$ on $\Y$, chosen as the counting measure for discrete outcomes and Lebesgue measure for continuous ones. We assume throughout that the joint distribution admits a conditional density of $Y$ given $X$ with respect to $\mu_Y$.

\begin{assumption}[Existence of conditional density]\label{ass:rn}
The joint distribution $\PXY$ is absolutely continuous with respect to
$\PX\otimes\mu_Y$, i.e.,
\[
\PXY \ll \PX\otimes\mu_Y
\]
\end{assumption}

Under~\cref{ass:rn}, the Radon--Nikodym derivative
\[
f := \frac{d\PXY}{d(\PX\otimes\mu_Y)}
\]
exists and defines a version of the conditional density of $Y$ given $X$ with respect to $\mu_Y$, in the sense that for $\PX$-almost every $x$ and all measurable $B\subseteq\Y$,
\[
\mathbb{P}(Y\in B\mid X=x)=\int_B f(x,y)\,\mu_Y(dy).
\]

\subsubsection{Piecewise-constant approximation on a partition}

Define $\nu := \PX\otimes\mu_Y$ and let $\bar{\Z}\subseteq\Z$ be a measurable subset with $\nu(\bar{\Z})<\infty$.
When $\nu(\Z)<\infty$, one may take $\bar{\Z}=\Z$. Let $\pi$ be a measurable
partition of $\bar{\Z}$ into cells $A=A_X\times A_Y$.

On a cell $A$ with $0<\nu(A)<\infty$, the best constant approximation of $f$ in
the sense of matching the $\PXY$-mass of the cell is obtained by choosing
$c_A\ge 0$ such that
\[
\int_A c_A\,\nu(dz)=\PXY(A).
\]
Solving gives the cell-average value
\begin{equation}
\label{eq:ca_improved}
c_A = \frac{\PXY(A)}{\nu(A)}.
\end{equation}
This yields the piecewise-constant estimator
\begin{equation}
\label{eq:cond_exp_improved}
f_{\pi}(z) =
\begin{cases}
\displaystyle \frac{\PXY(\pi[z])}{\nu(\pi[z])}, & z\in\bar{\Z},\\[0.6em]
0, & \text{otherwise.}
\end{cases}
\end{equation}
Equivalently, $f_{\pi}=\E_\nu[f\mid\sigma(\pi)]$ $\nu$-a.e.\ on $\bar{\Z}$, i.e.,
the conditional expectation of $f$ given the $\sigma$-algebra generated by the
partition.

Because $\nu=\PX\otimes\mu_Y$ factorizes over $X$ and $Y$, a product cell
$A=A_X\times A_Y$ satisfies
\[
\nu(A)=\PX(A_X)\,\mu_Y(A_Y),
\]
so \eqref{eq:ca_improved} can be written more explicitly as
\[
c_A=\frac{\PXY(A)}{\PX(A_X)\,\mu_Y(A_Y)}.
\]

\subsubsection{Empirical estimator from counts}

We estimate $\PXY(A)$ and $\PX(A_X)$ from the sample
$\D=\{(x_i,y_i)\}_{i=1}^N$. Recall the counts defined in Equations \eqref{eq:nxy} and \eqref{eq:nx}. Then $\hPXY(A)=n_{XY}(A)/N$ and $\hPX(A_X)=n_X(A)/N$, and the
plug-in estimate of $c_A$ becomes
\begin{equation}
\label{eq:chat_improved}
\hat c_A
=
\frac{\hPXY(A)}{\hPX(A_X)\,\mu_Y(A_Y)}
=
\frac{n_{XY}(A)}{n_X(A)\,\mu_Y(A_Y)},
\end{equation}
defined whenever $n_X(A)>0$ and $\mu_Y(A_Y)>0$. Substituting \eqref{eq:chat_improved}
into \eqref{eq:cond_exp_improved} yields a fully explicit estimator once a
partition $\pi_N$ is fixed.


\subsubsection{Truncation and practical estimator}

When $\mu_Y(\Y)=\infty$ (for instance, because $\Y$ has an unbounded
continuous component), we work on a data-dependent truncated domain
$\bar{\Y}_N$ and write $\bar{\Z}_N:=\X\times\bar{\Y}_N$. In the mixed case,
only the continuous outcome coordinates are truncated. The appendix
gives the explicit min--max box construction in~\eqref{eq:yn_improved} and shows
that it captures asymptotically all probability mass (\cref{lem:box_mass}).

Given a partition $\pi_N$ of $\bar{\Z}_N$, the estimator used by the tree is
the plug-in rule
\[
\fpart(z)
=
\begin{cases}
\displaystyle \frac{n_{XY}(\pi_N[z])}{n_X(\pi_N[z])\,\mu_Y(\pi_N[z]_Y)},
& z\in\bar{\Z}_N,\ n_X(\pi_N[z])>0,\\[0.8em]
0, & \text{otherwise.}
\end{cases}
\]
This is the empirical counterpart of~\eqref{eq:cond_exp_improved}; the
consistency analysis of this estimator is given in~\cref{sec:consistency}.

\subsubsection{Normalization}

For a fixed $x$, the map $y\mapsto\fpart(x,y)$ is piecewise constant over the sets
$\{A_Y: A\in\pi_N,\ x\in A_X\}$, and its normalizing factor depends on the
corresponding $X$-projections $A_X$. Unless these sets form a partition of
$\bar{\Y}_N$ with a common denominator, $\fpart(x,\cdot)$ need not integrate to one.
When a proper conditional density on $\bar{\Y}_N$ is required, we therefore use the normalized version
\begin{equation}\label{eq:density_normalized}
\fpartn(x,y)
=
\frac{\fpart(x,y)}
{\int_{\bar{\Y}_N}\fpart(x,y')\,d\mu_Y(y')},
\end{equation}
whenever the denominator is positive.
The unnormalized estimator $\fpart$ is analyzed first in~\cref{sec:consistency},
and \cref{cor:extrapolation_consistency} shows that this normalization step
preserves consistency.


The next section describes how to efficiently choose $\pi_N$ to optimize a conditional log-loss objective.

\subsection{Tree Construction Algorithm}\label{sec:tree_construction}

\begin{figure*}
    \centering
    \includegraphics[width=\columnwidth]{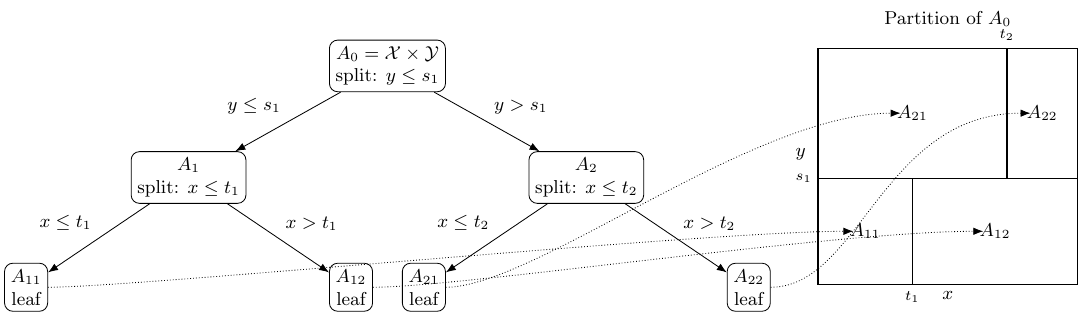}
    \caption{Illustration of a partition of the joint space $\Z=\X\times\Y$. Each leaf of the tree corresponds to a rectangular cell in the induced partition (right). For a fixed query $x=t$, the leaves whose $X$-projection contains $t$ form a histogram over $\Y$: the associated $Y$-intervals $A_Y$ are the bins, and the estimator is constant on each bin. In particular, for $x\le t_1$ the slice $x=t$ intersects two leaves, yielding a two-bin histogram defined by $A_{11}$ and $A_{12}$; for $t_1<x\le t_2$, it intersects the leaves $A_{21}$ and $A_{12}$. On each cell, the conditional density is estimated from empirical counts normalized by the $\Y$-volume $\mu_Y(A_Y)$.}
    \label{fig:tree_illustration}
\end{figure*}

We build a data-dependent partition $\pi_N$ of $\bar{\Z}_N=\X\times\bar{\Y}_N$
using a greedy and best-first tree procedure. Each leaf corresponds to a measurable product cell
\[
A = A_X \times A_Y,
\]
and the tree encodes a recursive refinement of such cells. Intuitively, splits
on $X$ refine regions of the covariate space, while splits on $Y$ refine the
outcome bins inside a fixed covariate region. This joint-space partition is
exactly what is needed by the estimator $\fpart$ defined above,
since it defines (for each $x$-region) a data-adaptive histogram on $Y$, as illustrated in~\cref{fig:tree_illustration}.

\subsubsection{Candidate splits and feasibility}

A candidate split acts on a single coordinate of $Z=(X, Y)$ and preserves the
product structure of the leaves:
\begin{align*}
A \;\longrightarrow\; \{A_l,A_r\},\qquad A = A_l \uplus A_r \\
A_l = A_{l,X}\times A_{l,Y},\;\; A_r = A_{r,X}\times A_{r,Y}.
\end{align*}
If we split along an $X$-coordinate, then $A_{l,Y}=A_{r,Y}=A_Y$ and $A_X$ is
partitioned into $A_{l,X}\uplus A_{r,X}$. If we split along a $Y$-coordinate,
then $A_{l,X}=A_{r,X}=A_X$ and $A_Y$ is partitioned into $A_{l,Y}\uplus A_{r,Y}$.

For continuous coordinates $z_\ell\in \mathbb{R}$, we use threshold splits
$z_\ell\le t$ vs.\ $z_\ell>t$. For categorical coordinates $z_\ell\in\Sigma$, we
use subset tests $z_\ell\in S$ vs.\ $z_\ell\notin S$ for
$\emptyset\subsetneq S\subsetneq\Sigma$. In the practical gain-based split
search, we only accept splits that keep both children populated and
well-defined for the density estimator, i.e.,
\begin{align*}
n_X(A_l)>0,\quad n_X(A_r)>0, \mu_Y(A_{l,Y})>0,\quad \mu_Y(A_{r,Y})>0.
\end{align*}
Clearly, when splitting on a $Y$-coordinate, $n_X(A_l)=n_X(A_r)=n_X(A)$.

Candidate gain-based splits producing a child with $n_X(A)=0$ or
$\mu_Y(A_Y)=0$ are discarded; if no admissible split exists for a leaf, it is
declared terminal. By definition, $\fpart$ uses the zero convention on cells
with $n_X(A)=0$, although such cells do not arise under these admissibility
constraints.

\subsubsection{Log-loss objective and split gain}

Given a partition $\pi$, recall the unnormalized estimator
$f_\pi(z)=\PXY(\pi[z])/\nu(\pi[z])$, where $\nu=\PX\otimes\mu_Y$.
We use its log-loss as a surrogate for conditional log-loss.
This objective decomposes into contributions from individual leaves,
so a split can be scored using only the parent and its two children:
\begin{align}\label{eq:pop_loss_tree_improved}
\mathcal L(\pi)
&:= -\E_{\PXY}\big[\log f_\pi(Z)\big] \\
&= -\sum_{A\in\pi}\PXY(A)\log\!\left(\frac{\PXY(A)}{\PX(A_X)\,\mu_Y(A_Y)}\right),
\nonumber
\end{align}
with conventions $0\log 0:=0$.

Consider splitting a leaf $A\in\pi$ into $A_l,A_r$ and let
$\pi' := (\pi\setminus\{A\})\cup\{A_l,A_r\}$. The population gain is the log-loss
reduction
\[
G(\pi',\pi) := \mathcal L(\pi)-\mathcal L(\pi')\;\;\ge 0,
\]
and can be written as a Jensen gap (App.~\ref{app:gain_props}).

In practice, we maximize the empirical gain obtained by replacing $\PXY$ and $\PX$
with their empirical counterparts.
\begin{align}\label{eq:empirical_gain}
\widehat G(\pi', \pi )
&=
\frac{n_{XY}(A_l)}{N}\log\!\left(\frac{n_{XY}(A_l)}{n_X(A_l)\,\mu_Y(A_{l,Y})}\right) \nonumber
\quad+
\frac{n_{XY}(A_r)}{N}\log\!\left(\frac{n_{XY}(A_r)}{n_X(A_r)\,\mu_Y(A_{r,Y})}\right)
\quad\\&-
\frac{n_{XY}(A)}{N}\log\!\left(\frac{n_{XY}(A)}{n_X(A)\,\mu_Y(A_Y)}\right),
\end{align}
where any term with $n_{XY}(\cdot)=0$ contributes $0$.

The empirical surrogate upper-bounds the training log-loss of the
normalized estimator. App.~\ref{app:gain_props} proves
this bound and the nonnegativity of the gains.

\subsubsection{Best-first growth}

The tree is grown under a global split budget $k_N$ by repeatedly selecting the
leaf--split pair with the largest empirical gain. Concretely, for each current
leaf $A$ we search over admissible one-coordinate splits, record its best gain,
and maintain the leaves in a priority queue so that only the affected entries
need to be recomputed after a split.

\subsubsection{Efficient split search and complexity}

A computational advantage of \eqref{eq:empirical_gain} is that it depends only on simple count statistics for candidate children. For a continuous coordinate $z_\ell$ and a leaf $A$, candidate thresholds are the midpoints between consecutive distinct values observed in the leaf. After sorting the relevant values within each leaf, we scan thresholds in increasing order and update the required counts using prefix sums, enabling evaluation of all thresholds at that coordinate in linear time, plus the cost of sorting. \Cref{alg:split_continuous} in~ App.~\ref{app:algorithm_continuous} gives one concrete implementation using two sorted index lists: one for the covariate samples in $A_X$ (to update $n_X$ when splitting on $X$), and another for the joint samples in $A$ (to update $n_{XY}$ and to handle splits on either $X$ or $Y$). 

For categorical coordinates with alphabet $\Sigma$, we avoid enumerating all $2^{|\Sigma|}-2$ subset tests by computing per-category count statistics inside the leaf, sorting categories by the score, and scanning the $|\Sigma|-1$ prefix
thresholds. App.~\ref{app:cat_subset_split} shows that this scan attains the best
subset split for a leaf and coordinate.

Let $d_Z$ be the number of splittable coordinates in $Z=(X, Y)$. With per-leaf sorting, evaluating all candidate splits at a leaf with $n_X(A)$ covariate points costs $O(d_Z\;n_X(A) \log n_X(A))$ in a straightforward implementation. \\[0.2em]
For categorical coordinates, since alphabet sizes are fixed, the sort-and-scan subset search adds only $O(n_X(A))$ work per leaf, which is smaller than the split search cost over continuous coordinates. Summed over the best-first growth, this leads to an overall time complexity on the order of $O(d_Z\, N\log N)$ with global presorting and maintained sorted indices, or $O(d_Z\, N\log^2 N)$ when sorting within each node for balanced trees, matching the scaling observed in practice for CART trees.

\subsection{\pf\ (bagging ensemble)}
To improve predictive stability and log-loss, we also consider an ensemble variant, \pf, obtained by bagging \pts. We fit $B$ trees independently, where tree $b$ is trained on a bootstrap sample $\D^{(b)}$ (and optionally using random feature subsampling at each split, as in Random Forests~\cite{breiman2001random}). Given per-tree conditional density estimates $\{\fpart^{(b)}\}_{b=1}^B$, the forest predictor is the averaged density
\[
\fpart^{\mathrm{F}}(x,y) := \frac{1}{B}\sum_{b=1}^B \fpart^{(b)}(x,y),
\]
which is then normalized as in~\cref{eq:density_normalized} to obtain $\fpartn^{\mathrm{F}}(x,y)$.

\section{Consistency}\label{sec:consistency}

In this section, we give sufficient conditions for $L^1(\nu)$-consistency of the
piecewise-constant estimator in terms of the induced partition sequence. Hence,
the result applies to any data-driven tree construction whose partitions
satisfy the corresponding complexity and shrinkage conditions. For Partition Tree, we must verify these conditions for the chosen partitioning strategy.

The consistency of data-driven partitions for joint density estimation was studied by \cite{lugosi1996consistency}. We adapt those results to conditional density estimation. To state the assumptions, we introduce the required complexity measures. The maximum cell count of a family of partitions $\mathcal{A}$ is
\begin{equation*}
    m(\mathcal A) = \sup_{\pi \in \mathcal{A}}|\pi|
\end{equation*}
and the maximum number of distinct $\X$-projections induced by a partition family is
\begin{equation}
    m_X(\mathcal A_N)
:=\sup_{\pi\in\mathcal A_N}\big|\{A_X:\ A\in\pi\}\big|.
\end{equation}
Denote by $\Delta(\mathcal{A}, \D)$ the number of distinct partitions of the finite set $\D$ induced by partitions in $\mathcal{A}$. The growth function $\Delta^*_N(\mathcal{A})$ is defined by
\begin{equation*}
\Delta^*_N(\mathcal{A}) := \max\limits_{\D \in \mathcal Z^N} \Delta(\mathcal{A}, \D).
\end{equation*}
We assume that the class of $X$-projections of leaves has finite VC dimension.
In our implementation, it consists of axis-aligned rectangles in $\mathbb R^{\dxc}$ times categorical cylinder sets over finite alphabets. This class has a finite VC dimension by standard results on VC bounds for rectangles and finite product classes \citep{van2009note}.
We also require that the partitioning rule returns only leaves with positive $\hPX$ and $\mu_Y$ mass, which is enforced by the minimum-population and positive-volume hyperparameters used when accepting splits in the tree construction.

\begin{assumption}[Admissible leaves]\label{ass:admissible_leaves}
The partitioning rule returns only leaves $A\in\pi_N$ satisfying
$\hPX(A_X)>0$ and $\mu_Y(A_Y)>0$, almost surely. In the tree implementation,
this is enforced by the minimum-population and positive-volume hyperparameters
used when accepting splits.
\end{assumption}

\Cref{the:consistency_1} gives the main sufficient conditions for any
data-driven partition sequence of $\Z$ satisfying these assumptions to be
$\lnu$-consistent. The proof is deferred to App.~\ref{app:consistency}.

To keep the theorem statement self-contained, let $\bar{\Y}_N\subseteq\Y$
denote a measurable truncation with $\mu_Y(\bar{\Y}_N)<\infty$ and
$\PXY(Y\notin\bar{\Y}_N)\to 0$ almost surely, and write
$\bar{\Z}_N:=\X\times\bar{\Y}_N$. For a cell $A$, let
\[
\diamh(A):=\diamc(A)+\sum_k q_k(A),
\]
where $\diamc(A)$ is the diameter of the continuous coordinates after the
bounded monotone transform described in App.~\ref{app:mixed_type_setup}, and
$q_k(A)$ measures how many categories remain unresolved along categorical
coordinate $k$.

\begin{restatable}{theorem}{ConsistencyOne}
\label{the:consistency_1}
    Let $\D_N = \{(X_i, Y_i)\}_{i=1}^N = \{\Z_i\}_{i=1}^N$ be a set of observations
    belonging to $\Z := \X \times \Y$ with joint distribution $\PXY \ll \nu$.
    Let $\Pi=\{\pi_1, \dots, \pi_N\}$ be a
    partitioning scheme for $\bar{\Z}_N$, and let $\A_N$ be the
    collection of partitions associated with the rule $\pi_N$. Denote by
    $f(X, Y)$ the conditional density of $Y$ given $X$ and the estimate as
    in~\cref{eq:density_formula_improved}. Assume that the class $\mathcal C_X$
    of possible $X$-projections of leaves has finite VC dimension and that
    \cref{ass:admissible_leaves} holds. Assume further:
    \begin{enumerate}
    \item $N^{-1}m(\A_N) \tendsto 0$
    \item $N^{-1} \log \Delta^*_N(\A_N)  \tendsto 0$   
    \item There exists a sequence $\gamma_N \xrightarrow{N \to \infty }{0}$
    such that
    $\PXY(\{z \in \bar{\Z}_N : \diamh(\pi_N[z]) > \gamma_N\})
    \xrightarrow{N \to \infty }{0}$ almost surely.
    \item $m_X(\A_N)\sqrt{\frac{\log N}{N}}\to 0.$

    \end{enumerate}
    Then
    \[
    \| \fpart - f\|_{L^1(\nu)} \longrightarrow 0 \quad\text{as } N\to\infty,
    \]
    almost surely.
\end{restatable}

\begin{proof} See~\cref{app:consistency} for the proof.\end{proof}

\Cref{the:consistency_1} gives sufficient conditions on the sequence
of estimated partitions. For Partition Tree,
App.~\ref{app:discussion_assumptions} verifies conditions 1, 2, and 4
when the split budget satisfies $k_N\to\infty$ and
$k_N=o(\sqrt{N/\log N})$. The shrinkage condition 3 additionally
depends on the split-selection rule.
\Cref{prop:exploration_splits} establishes it for an
exploration-augmented construction without additional structural
assumptions on the data distribution. For the greedy rule used in the experiments,
\Cref{prop:greedy_margin} establishes shrinkage, and hence
consistency, under a gain-margin assumption: whenever a
non-negligible probability mass remains in large cells, the
empirically selected splits must achieve sufficient population
gain. The consistency guarantee for the experimental algorithm
is therefore conditional on this additional assumption.

\section{Experiments}\label{sec:experiments}

We conduct experiments to answer the following questions:
\begin{compactenum}[(1)]
    \item How does the proposed algorithm compare to the traditional CART decision tree and its probabilistic variants?
    \item Can the proposed tree handle both homoscedastic and heteroscedastic noise?
    \item How does the proposed tree algorithm behave under correlated, noisy feature duplication?
\end{compactenum}
Our experiments assess the performance of the proposed greedy gain-based splitting strategy on both classification and regression tasks. We focus on comparisons against the CART decision tree algorithm~\cite{breiman2017classification}, implemented in \scikit~\cite{scikit-learn}. Additionally, we compare the results against Random Forest implemented in~\scikit.

We use five-fold cross-validation, with a nested train-validation split within each fold for hyperparameter tuning. Hyperparameter tuning is performed using Optuna~\cite{akiba2019optuna}, with a budget of 200 trials and a time limit of 20 minutes per model. For CDTree, hyperparameters are not tuned, following the recommendations of the original implementation. Model selection is based on log-loss evaluated on the validation data. Additional details regarding hyperparameters and experimental settings are provided in~\Cref{app:experiments}. All experiments were conducted using our implementation of \pt\ and \pf. 
All experiments were conducted on an Apple MacBook equipped with an Apple M4 chip.

\paragraph{Classification}

\Cref{tab:ClassifLogLoss} reports the log-loss obtained by \pt\ and \pf\ in comparison with their respective baselines.
Among single-tree methods, Partition Tree has lower mean classification log-loss than CART on 6 of 9 datasets and a better average rank (1.33 vs. 1.67).
Among bagging methods, \pf\ has a better average rank than Random Forest (1.44 vs. 1.56) and lower mean log-loss on five of the nine datasets.
%
%
\begin{table}
\centering
\caption{Log-loss results over 5-fold cross-validation for single tree and bagging tree methods on classification and regression tasks (mean ± std). Lower is better. Best for single tree and bagging in bold. The last row represents the average ranking of each method in its group. \underline{\textit{Underline}} indicates statistical ties according to a paired t-test at the 5\% significance level (run separately within single tree and bagging groups)}\label{tab:results}
\footnotesize{
  \begin{subtable}{0.8\linewidth}
     \caption{Classification tasks}\label{tab:ClassifLogLoss}
     \begin{adjustbox}{width=\linewidth}
        \begin{tabular}{lcccc}
\toprule
 & \multicolumn{2}{c}{Single tree} & \multicolumn{2}{c}{Bagging} \\
\cmidrule(lr){2-3} \cmidrule(lr){4-5}
 & Partition Tree & CART (Logloss) & Partition Forest & RF (Logloss) \\
\midrule
iris & \underline{\textit{0.46 $\pm$ 0.38}} & \textbf{0.46 $\pm$ 0.46} & \underline{\textit{0.50 $\pm$ 0.43}} & \textbf{0.15 $\pm$ 0.08} \\
breast & \textbf{0.64 $\pm$ 0.05} & \underline{\textit{0.70 $\pm$ 0.26}} & \underline{\textit{0.58 $\pm$ 0.02}} & \textbf{0.56 $\pm$ 0.04} \\
wine & \textbf{1.02 $\pm$ 0.06} & \underline{\textit{1.47 $\pm$ 0.50}} & \textbf{0.86 $\pm$ 0.04} & \underline{\textit{0.88 $\pm$ 0.11}} \\
digits & \textbf{0.45 $\pm$ 0.04} & 1.48 $\pm$ 0.22 & \textbf{0.29 $\pm$ 0.01} & \underline{\textit{0.29 $\pm$ 0.01}} \\
spam & \textbf{0.22 $\pm$ 0.04} & 0.29 $\pm$ 0.06 & \underline{\textit{0.17 $\pm$ 0.03}} & \textbf{0.16 $\pm$ 0.02} \\
support2 & \underline{\textit{0.28 $\pm$ 0.02}} & \textbf{0.28 $\pm$ 0.02} & \textbf{0.24 $\pm$ 0.02} & \underline{\textit{0.25 $\pm$ 0.01}} \\
letter & \textbf{0.62 $\pm$ 0.17} & 1.42 $\pm$ 0.06 & 0.49 $\pm$ 0.01 & \textbf{0.28 $\pm$ 0.01} \\
bank & \textbf{0.22 $\pm$ 0.01} & 0.22 $\pm$ 0.01 & \textbf{0.20 $\pm$ 0.01} & \underline{\textit{0.20 $\pm$ 0.01}} \\
adult & 0.93 $\pm$ 0.01 & \textbf{0.92 $\pm$ 0.01} & \textbf{0.90 $\pm$ 0.01} & \underline{\textit{0.90 $\pm$ 0.01}} \\
\midrule
Avg. rank & 1.33 & 1.67 & 1.44 & 1.56 \\
\bottomrule
\end{tabular}
     \end{adjustbox}
  \end{subtable}

  \medskip

  \begin{subtable}{0.8\linewidth}
     \caption{Regression tasks}\label{tab:RegLogLoss}
     \begin{adjustbox}{width=\linewidth}
         \begin{tabular}{lcccccc}
\toprule
 & \multicolumn{4}{c}{Single tree} & \multicolumn{2}{c}{Bagging} \\
\cmidrule(lr){2-5} \cmidrule(lr){6-7}
 & Partition Tree & CART & CADET & CDTree & Partition Forest & RF \\
\midrule
diabetes & \underline{\textit{5.65 $\pm$ 0.11}} & 17.93 $\pm$ 2.49 & \underline{\textit{5.59 $\pm$ 0.14}} & \textbf{5.58 $\pm$ 0.06} & \textbf{5.49 $\pm$ 0.10} & 14.21 $\pm$ 2.43 \\
boston & \underline{\textit{2.90 $\pm$ 0.17}} & 8.39 $\pm$ 3.33 & \underline{\textit{2.89 $\pm$ 0.18}} & \textbf{2.88 $\pm$ 0.14} & \textbf{2.53 $\pm$ 0.13} & 4.60 $\pm$ 1.08 \\
energy & 2.36 $\pm$ 0.18 & \underline{\textit{1.67 $\pm$ 0.04}} & \underline{\textit{10.81 $\pm$ 18.62}} & \textbf{1.55 $\pm$ 0.16} & 2.06 $\pm$ 0.07 & \textbf{1.55 $\pm$ 0.03} \\
concrete & \underline{\textit{3.70 $\pm$ 0.06}} & 7.45 $\pm$ 0.90 & \textbf{3.52 $\pm$ 0.17} & \underline{\textit{3.66 $\pm$ 0.06}} & \textbf{3.36 $\pm$ 0.03} & 3.86 $\pm$ 0.40 \\
kin8nm & \underline{\textit{-0.18 $\pm$ 0.02}} & 9.42 $\pm$ 0.15 & \textbf{-0.22 $\pm$ 0.04} & \underline{\textit{-0.18 $\pm$ 0.01}} & \textbf{-0.45 $\pm$ 0.02} & 3.30 $\pm$ 0.12 \\
air & \textbf{2.12 $\pm$ 0.11} & 13.49 $\pm$ 1.00 & \underline{\textit{13.78 $\pm$ 17.54}} & 6.42 $\pm$ 0.00 & \textbf{2.02 $\pm$ 0.10} & 11.35 $\pm$ 0.88 \\
power & \textbf{2.87 $\pm$ 0.02} & 3.13 $\pm$ 0.10 & \underline{\textit{2.90 $\pm$ 0.10}} & 2.93 $\pm$ 0.03 & \textbf{2.60 $\pm$ 0.02} & \underline{\textit{2.63 $\pm$ 0.06}} \\
naval & -4.38 $\pm$ 0.15 & 4.69 $\pm$ 0.22 & \textbf{-4.78 $\pm$ 0.08} & -4.29 $\pm$ 0.03 & -4.72 $\pm$ 0.01 & \textbf{-4.98 $\pm$ 0.01} \\
california & 0.69 $\pm$ 0.02 & 7.07 $\pm$ 0.36 & 0.88 $\pm$ 0.16 & \textbf{0.58 $\pm$ 0.02} & \textbf{0.42 $\pm$ 0.02} & 3.11 $\pm$ 0.14 \\
protein & 2.22 $\pm$ 0.02 & 12.31 $\pm$ 0.20 & 3.13 $\pm$ 0.22 & \textbf{2.18 $\pm$ 0.02} & \textbf{2.10 $\pm$ 0.02} & 6.19 $\pm$ 0.04 \\
\midrule
Avg. rank & 2.30 & 3.70 & 2.30 & 1.70 & 1.20 & 1.80 \\
\bottomrule
\end{tabular}
     \end{adjustbox}
   \end{subtable}
}
\end{table}
\paragraph{Regression Task}
%
%
We compare \pt\ against probabilistic tree regression baselines: CDTree~\cite{yang2024conditional} and CADET~\citep{cousins2019cadet}.
Since CART decision trees do not natively support probabilistic regression, we adopt the residual-based approach implemented in \texttt{skpro}~\cite{gressmann2018probabilistic} for CART and Random Forest. It assumes normally distributed residuals and estimates predictive variance using an inner train-test split.

The log-loss results are summarized in~\Cref{tab:RegLogLoss}. Among single-tree methods, CDTree achieves the best average ranking. However, its high computational complexity limits scalability and can make integration into bagging frameworks impractical, as shown in the runtime comparison in~\Cref{fig:runtime} in~\Cref{app:runtime}. Inference times for the single-tree regression methods are also reported in~\Cref{tab:inference_time} in~\Cref{app:runtime}.
Comparing \pt\ with CADET, our method outperforms CADET on Air, Power, California, and Protein, all of which are among the larger datasets in the benchmark suite. This pattern suggests that \pt\ may benefit more from larger sample sizes than the parametric baseline.
For bagging methods, \pf\ has lower log-loss than Random Forest on 8 of 10 regression datasets. 
However, Random Forest performs better on the point-prediction metrics, with better average ranks for classification accuracy (1.33 vs. 1.67) 
and regression RMSE (1.10 vs. 1.90; see \Cref{tab:ClassifAccuracy,tab:RegRMSESpread} in~\Cref{app:experiments}). Thus, \pf’s advantage in these experiments is in probabilistic prediction, while Random Forest is the stronger choice when point prediction is the priority.

We extend the regression evaluation with two complementary comparisons. First, we compare against KCDE, a non-tree conditional density estimator: \pf\ achieves the best average rank (1.20), followed by KCDE (2.10) and \pt\ (2.70). Second, since CDTree uses MDL-based model selection without external hyperparameter tuning~\cite{yang2024conditional}, we compare it with fixed configurations of \pt\ and \pf\ across all ten datasets. Under this no-tuning protocol, \pf\ again ranks first (1.20), followed by CDTree (2.10) and \pt\ (2.70). Full settings and per-dataset results are provided in~\Cref{app:additional_regression_comparisons}.



\paragraph{Robustness against noise and redundant features}
To evaluate robustness to correlated and redundant features, we augment the feature matrix $X \in \mathbb{R}^{N \times \dx}$ with $k$ additional noisy columns. Each added feature is generated by copying a randomly selected original column $X_{c_i}$ and adding Gaussian noise $\epsilon \sim \mathcal{N}(0, \sigma_i^2)$. The noise scale is set to the mean absolute value of the selected column, $\sigma_i = \frac{1}{N} \sum_j |X_{j, c_i}|$, which ensures that the perturbation is proportional to the feature's magnitude while preserving substantial correlation with the original predictor. 
%
%
We evaluated robustness on the Concrete Compressive Strength dataset using five-fold cross-validation. The results are summarized in~\cref{fig:noisy_features}. CDTree maintains a stable negative log-likelihood as the number of features increases, indicating robustness to 
correlated feature duplication. In contrast, \pt\ showed a gradual degradation in the negative log-likelihood, although its performance variance across folds remains lower than that of CADET. One possible explanation is that CDTree's MDL objective penalizes model complexity and its look-ahead scoring evaluates whether a split supports a useful conditional histogram. \pt\ instead selects the largest empirical gain among candidate splits; adding correlated copies creates more opportunities for a noisy gain to outrank that of the original feature. CADET also degrades in this experiment, consistent with this possible effect of gain-based split selection.


\paragraph{Homoscedastic and heteroscedastic noise}
We study robustness to label noise by constructing semi-synthetic datasets in which we perturb the target values. For the homoscedastic setting, we add Gaussian noise with standard deviation $\sigma(\lambda) = \lambda\cdot \frac{1}{N} \sum_i |y_i|$, with $\lambda \in \{0.1, 0.5, 1.0, 1.5\}$. For the heteroscedastic case, we add sample-dependent Gaussian noise with $\sigma_i(\lambda) = \lambda |y_i|$, using the same values of $\lambda$. \Cref{fig:noisehomoscedastic,fig:noiseheteroscedastic} shows that \pt\ degrades with noise at a rate comparable to CDTree in both settings, while CADET exhibits higher variability across folds.
\begin{figure}
    \centering
    \begin{subfigure}[c]{0.31\linewidth}
        \includegraphics[width=\linewidth]{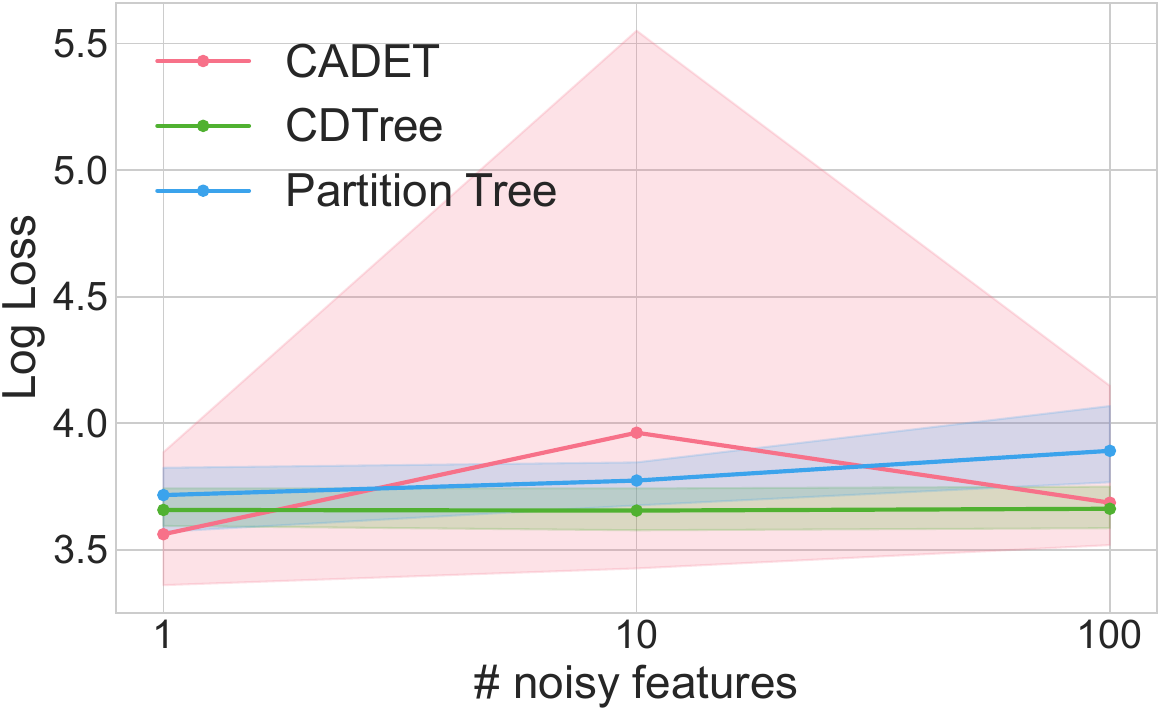}
        \caption{Noisy feature duplication}\label{fig:noisy_features}
    \end{subfigure}
    \begin{subfigure}[c]{0.31\linewidth}
        \includegraphics[width=\linewidth]{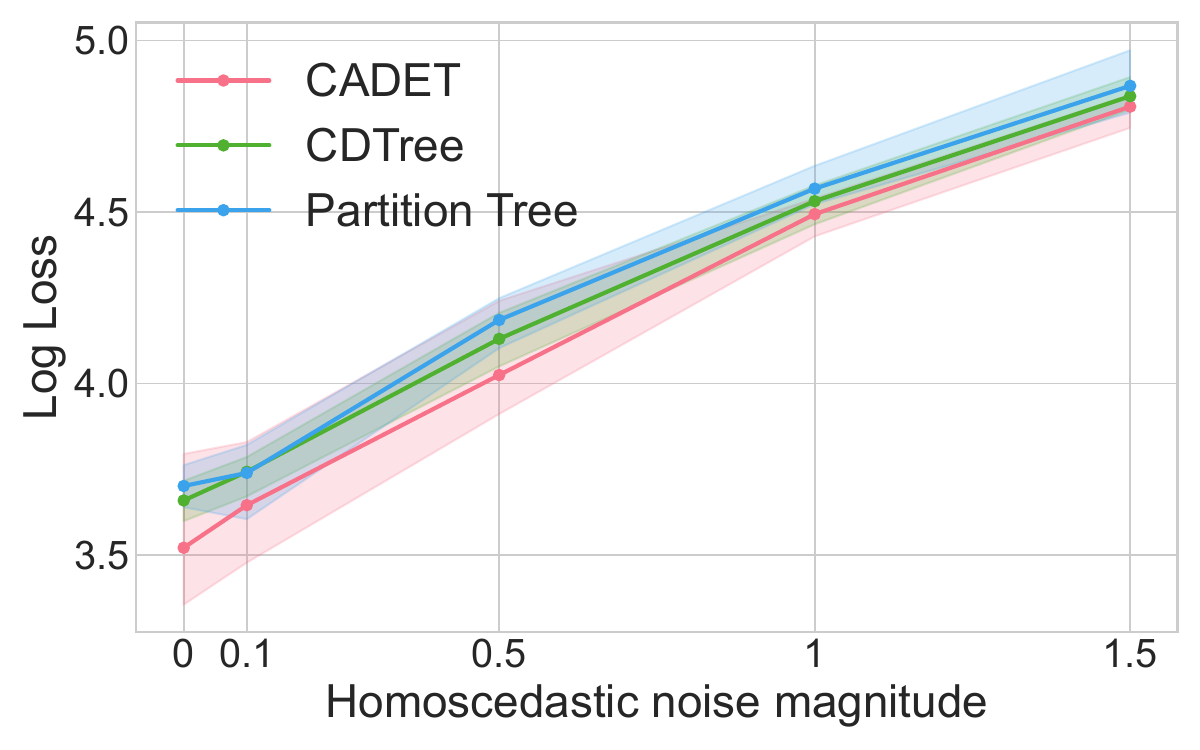}
        \caption{Homoscedastic noise}\label{fig:noisehomoscedastic}
    \end{subfigure}
    \begin{subfigure}[c]{0.31\linewidth}
        \includegraphics[width=\linewidth]{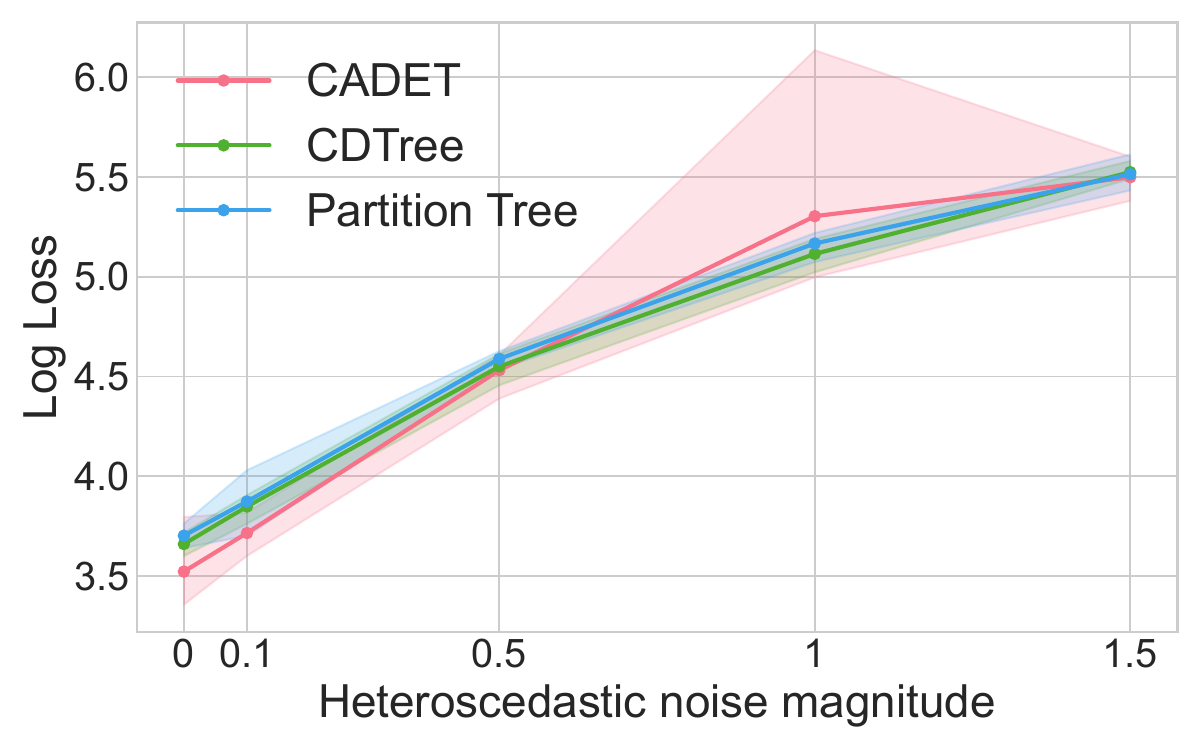}
        \caption{Heteroscedastic noise}\label{fig:noiseheteroscedastic}
    \end{subfigure}
    \caption{Performance of probabilistic tree models under different types of noise. Shaded bands indicate the minimum and maximum negative log-likelihood across five cross-validation folds.}\label{fig:noise}
\end{figure}


\section{Conclusion}
We introduced \emph{\pt}, a decision-tree framework for conditional density estimation over general outcome spaces, including continuous, categorical, and mixed-type targets. The method is grounded in a measure-theoretic formulation in which conditional densities are defined as a Radon-Nikodym derivative with respect to the dominating measure $\PP_X\otimes\mu_Y$. This yields a unified piecewise-constant estimator that naturally covers both classification and regression within a single framework. 

\pt\ is learned via a greedy, best-first procedure that optimizes a conditional log-loss surrogate objective on data-adaptive partitions of the joint covariate-outcome space. We isolate general sufficient conditions for $L^1(\nu)$-consistency and show that the experimental greedy rule meets them under a gain margin condition; without that condition, our algorithm-specific guarantee is for the exploration-augmented rule. A bagging extension, \emph{\pf}, averages conditional densities across trees and improves predictive stability and log-loss performance.

Empirically, the proposed models demonstrate competitive probabilistic performance across a diverse set of classification and regression benchmarks. \pf\ typically achieves lower log-loss than CART-based trees and often improves over Random Forests, while providing accurate predictive densities and competitive point predictions in regression tasks. Experiments also indicate robustness to heteroscedastic noise and characterize behavior under correlated feature duplication.

\subsection*{Limitations and further work}

Several directions for future work remain open. From an algorithmic perspective, integrating \pf\ into boosting frameworks and developing pruning strategies tailored to conditional density objectives may further improve computational efficiency and predictive performance. More broadly, the ability of \pt\ to handle heterogeneous and mixed-type outcomes makes it a natural candidate for integration into causal inference pipelines, where flexible conditional density estimation plays a central role.

\section*{Acknowledgements.}
The authors thank the reviewers for their insightful and constructive feedback, which has substantially improved the quality of this work. AL acknowledges support from the French National Research Agency (ANR) under Grant No. ANR-23-CPJ1-0099-01. The authors declare no competing financial interests.

\bibliographystyle{unsrtnat}
\bibliography{references}


\appendix
\crefalias{section}{appendix}
\clearpage

\begin{center}
  \rule{\linewidth}{0.4pt}\\[0.6em]
  {\Large\bfseries Appendix}\\[0.6em]
  \rule{\linewidth}{0.4pt}
\end{center}
\startcontents[appendix]

\printcontents[appendix]
  {l}
  {1}
  {}

\onecolumn
\newpage

\section{Estimator and Objective}\label{app:estimator_details}

\subsection{Mixed-Type Covariates and Outcomes}\label{app:mixed_type_setup}
We allow $\X$ and $\Y$ to contain a mixture of coordinate types.
\begin{align*}
X&=(X^{(c)},X^{(k)})\in \Xc\times\Xk=: \X
\\
Y&=(Y^{(c)},Y^{(k)})\in \Yc\times\Yk=: \Y
\end{align*}
where,
\begin{compactenum}[(a)]
\item $\Xc\subseteq\mathbb R^{\dxc}$ and $\Yc\subseteq\mathbb R^{\dyc}$ are continuous coordinates;
\item $\Xk=\prod_{j=1}^{\dxk}\Sigma_{x,j}$ and $\Yk=\prod_{j=1}^{\dyk}\Sigma_{y,j}$ are categorical coordinates with alphabets $\Sigma_{x,j},\Sigma_{y,j}$ (assumed finite).
\end{compactenum}

Cells are factorized coordinate-wise, with internal splits for continuous coordinates, and subset splits for categorical coordinates.\\
We introduce some notions that will be important in the consistency theorem.
Let $\mathcal C_Z$ index the continuous coordinates of $Z=(X,Y)$. For each
continuous coordinate $r\in\mathcal C_Z$, fix a continuous strictly
bounded monotone transform $\psi_r:\mathcal Z_r\to(0,1)$. Define
\[
\dcont(z,z'):=\sum_{r\in\mathcal C_Z}|\psi_r(z_r)-\psi_r(z'_r)|,
\qquad
\diamc(A):=\sup_{z,z'\in A}\dcont(z,z').
\]
For a categorical coordinate $k$, let $S_k(A)\subseteq\Sigma_k$ be the set of
categories allowed by the cell $A$ along that coordinate and define
\[
q_k(A):=
\begin{cases}
\frac{|S_k(A)|-1}{|\Sigma_k|-1}, & |\Sigma_k|>1,\\
0, & |\Sigma_k|\le 1.
\end{cases}
\]
The diameter used in the shrinkage condition is
\[
\diamh(A) = \diamc(A) + \sum_{k} q_k(A).
\]
It is finite even for unbounded continuous coordinates and dominates the
product metric diameter used in the proof.

\subsubsection{Truncation of unbounded outcome spaces}\label{app:truncation_setup}

If $\mu_Y(\Y)=\infty$ (e.g., $\Y$ contains an unbounded continuous component), then $\nu(A)=\PX(A_X)\mu_Y(A_Y)$ may be infinite unless $A_Y$ has finite $\mu_Y$-measure. To ensure well-defined denominators while still covering most of the probability mass, we restrict the estimator to a data-dependent truncated outcome domain $\bar{\Y}_N$ with $\mu_Y(\bar{\Y}_N)<\infty$, and define
$\bar{\Z}_N:=\X\times\bar{\Y}_N$.

For $N$ samples $\{y_i^{(c)}\}_{i=1}^N$ in the continuous part $\Yc\subseteq\mathbb{R}^{\dyc}$, define the axis-aligned box
\begin{align}
\label{eq:yn_improved}
\bar{\Y}_N^{(c)}
:=
\prod_{j=1}^{\dyc}
\big[\,\underline y_{N,j}-\delta_N,\ \overline y_{N,j}+\delta_N\,\big],
\\
\underline y_{N,j}:=\min_{1\le i\le N} y_{i,j},\ \
\overline y_{N,j}:=\max_{1\le i\le N} y_{i,j}, \nonumber
\end{align}
and set $\bar{\Y}_N := \bar{\Y}_N^{(c)}\times\Y^{(k)}$ for the mixed case. The
padding $\delta_N>0$ (typically $\delta_N\downarrow 0$) avoids boundary effects
and ensures $\mu_Y(\bar{\Y}_N)<\infty$.

Given a data-dependent partition $\pi_N$ of $\bar{\Z}_N$, our estimator is
\begin{equation}\label{eq:density_formula_improved}
\fpart(z) =
\begin{cases}
\displaystyle
\frac{n_{XY}(\pi_N[z])}{n_X(\pi_N[z])\,\mu_Y(\pi_N[z]_Y)},
& z\in\bar{\Z}_N,\ n_X(\pi_N[z])>0,\\[0.9em]
0,& \text{otherwise.}
\end{cases}
\end{equation}
We assume that $\delta_N\downarrow 0$ and that the truncation captures asymptotically all the probability mass, \ie\
\[
\PXY (Y \notin \bar{\Y}_N) \longrightarrow 0 \quad\text{almost surely as } N\to\infty.
\]
We assume a deterministic boundary convention (\eg\ half-open intervals) so that partition cells form a disjoint cover up to $\mu_Y$-null cells.

The following lemma is used to ensure consistency between the tree algorithm and the truncation of $\Y$.

\begin{lemma}[Mass capture of the empirical min--max box]
\label{lem:box_mass} Let $(\Omega, \mathcal{A}, \PP)$ be a probability triple.
Let $\omega \in \Omega$, $Y\in \Y$, with $\Y := \mathbb R^{\dyc}$ have distribution $\PP_Y$ and let
$Y_1,Y_2,\dots$ be i.i.d.\ copies of $Y$.
Fix any deterministic padding sequence $\delta_N\ge 0$.
Define for each coordinate $j$:
\[
\underline y_{N,j}:=\min_{1\le i\le N} (Y_i)_j,\qquad
\overline y_{N,j}:=\max_{1\le i\le N} (Y_i)_j,
\]
and the random axis-aligned box
\[
\bYN := \prod_{j=1}^{\dyc}\big[\underline y_{N,j}-\delta_N,\ \overline y_{N,j}+\delta_N\big].
\]

Then
\[
\PP_Y(\bYN^c)\ \xrightarrow{N\to\infty}{}\ 0
\qquad\text{almost surely.}
\]
\end{lemma}

\newcommand{\ymin}{y^-_{N,j}}
\newcommand{\ymax}{y^+_{N,j}}
\newcommand{\qmin}{q^-_j}
\newcommand{\qmax}{q^+_j}
\newcommand{\badeventmin}{E^-_{N,j}}
\newcommand{\badeventmax}{E^+_{N,j}}

\begin{proof}
    Take $\epsilon \in (0, 1)$ and set $\alpha = \frac{\epsilon}{2d_y}$. For each coordinate $j \in \{1, \dots, \dyc\}$ define the marginal Cumulative Distribution Function (CDF):
    \[
    F_j(t) := \PP(Y_j \leq t)
    \]
    And the $\alpha$ and $1-\alpha$ quantiles:

    \[
    \qmin = \inf\{t : F_j(t) \geq \alpha\},\quad \qmax = \inf\{t : F_j(t) \geq 1 - \alpha \}
    \]

    Consider the minimum and maximum observed values:

    \[
    \ymin = \min_{1\leq i \leq N}(Y_i)_j, \quad  \ymax = \max_{1\leq i \leq N}(Y_i)_j
    \]

    Consider the bad events:

    \[
    \badeventmin := \{\ymin > \qmin\},\quad \badeventmax := \{\ymax < \qmax\}
    \]

    From independence of samples, $\PP(\badeventmin) = \PP(\ymin > \qmin) = \PP(Y_j > \qmin)^N= (1 - F_j(\qmin))^N \leq (1 - \alpha)^N$ and $\PP(\badeventmax) = \PP(\ymax < \qmax) = (\PP(Y_j < \qmax ))^N = (\lim_{t \uparrow \qmax }F_j(t))^N \leq (1-\alpha)^N$. By Borel-Cantelli, since $\sum_{N\geq1} \PP(\badeventmin) < \infty$ and $\sum_{N\geq1} \PP(\badeventmax) < \infty$:

    \[
     \PP(\badeventmin \text{ i.o.}) = 0, \quad \PP(\badeventmax \text{ i.o.}) = 0
    \]
    Therefore, for any given $j$, there exists $N_{0,j}(\omega)$ such that $\ymin \leq \qmin$ and $\ymax \geq \qmax$ for all $N \geq N_{0,j}(\omega)$ almost surely. Take $N_0(\omega) = \max_jN_{0,j}(\omega)$. For all $ N \geq N_0(\omega)$:

    \[
    \prod_{j=1}^{\dyc}[\qmin, \qmax] \subseteq \bYN
    \]

    And:

    \[
    \PP\!\big(\bYN^c\big)
    = \PY\!\Big(\Big(\prod_{j=1}^{\dyc}[\ymin-\delta_N,\ymax+\delta_N]\Big)^c\Big)
    \le
    \PY\!\Big(\Big(\prod_{j=1}^{\dyc}[\qmin,\qmax]\Big)^c\Big)
    \]

    Since

    \[
    (\prod_{j=1}^{\dyc}[\qmin, \qmax])^c  \subseteq \bigcup_{j=1}^{\dyc}(\{Y_j <\qmin\} \cup\{Y_j >\qmax\})
    \]

    we use the union bound:

    \[
    \PP(\bYN(\omega)^c)\leq \PY(\prod_{j=1}^{\dyc}[\qmin, \qmax]^c )\leq \sum_{j=1}^{\dyc}(\PY(Y_j < \qmin) + \PY(Y_j > \qmax)) \leq  \sum_{j=1}^{\dyc}2\alpha = \epsilon
    \]

    Since $\epsilon$ is arbitrary, we conclude $\PP_Y(\bYN^c)\to 0$ almost surely.

\end{proof}

\Cref{lem:box_mass} applies directly to continuous $\Y$. For categorical coordinates, we do not truncate, so in the mixed case, the construction simply combines the continuous min--max box with the full categorical outcome space. The consistency proof in Appendix~\ref{app:consistency} therefore studies the estimator on $\bZN$ while the truncation error vanishes asymptotically.

\subsection{Properties of the Log-loss Gain}
\label{app:gain_props}
We first show that the population gain of a candidate split is a Jensen gap, related to child averages. Define $\varphi(u) = u\log u$. For a leaf $A \in \pi$ and split $A \rightarrow \{A_l, A_r\}$. Let $\pi'$ be the partition after splitting $A$. The population loss of a partition $\pi$ can be expressed as:
\[
\Lo(\pi) = - \E_{\PXY}[\log f_\pi] = - \E_{\nu}[f_\pi\,\log f_\pi]
\]
so the population gain can also be written:
\begin{align}
\label{eq:gain_decomp}
G(\pi', \pi) &= G(\{A_l, A_r\}, \{A\}) \\ &=
\nu(A_l)\varphi(c_{A_l})  + \nu(A_r)\varphi(c_{A_r})  - \nu(A)\varphi(c_A) \\
&= \left(\frac{\nu(A_l)}{\nu(A)}\varphi(c_{A_l})  + \frac{\nu(A_r)}{\nu(A)}\varphi(c_{A_r})  - \varphi(c_A)\right)\, \nu(A)
\end{align}
where $c_{A} = \frac{1}{\nu(A)}\int_Afd\nu$.

From this identity, we derive the following proposition: the gain is a Jensen Gap.

\begin{proposition}[Population gain is a Jensen Gap]
\label{prop:jensen_gap}
    Consider a cell $A$ and two children $A_l$ and $A_r$, $A = A_l \cup A_r$. Let $\nu_A = \nu(\cdot \cap A)/\nu(A)$ be a normalized restriction of the measure $\nu$ to $A$ and let $\mathcal G := \sigma(\{A_l, A_r\})$ be the $\sigma$-algebra generated by the split. Assume $\nu(A), \nu(A_l), \nu(A_r) > 0$. The gain can be expressed by the following Jensen gap times the measure $\nu(A)$:

    \[
    G(\pi',\pi) = \left( \E_{\nu_A}[\varphi(X)] - \varphi(\E_{\nu_A}[X]) \right)  \cdot \nu(A)
    \]

    with $X(z) = \E_{\nu_A}[ f \mid \mathcal G] (z)$, $z \sim \nu_A$.
\end{proposition}

\begin{proof}
Define the normalized restriction $\nu_A(B):=\nu(B\cap A)/\nu(A)$ for measurable $B$.
Then
\[
c_A=\frac{1}{\nu(A)}\int_A f\,d\nu=\int_A f\,d\nu_A=\E_{\nu_A}[f].
\]
Let $\mathcal G:=\sigma(\{A_l,A_r\})$. Since $\mathcal G$ has atoms $A_l$ and $A_r$, the conditional expectation $X(z):=\E_{\nu_A}[f\mid\mathcal G](z)$ is $\mathcal G$-measurable, hence constant on each child: $X=u_l\1_{A_l}+u_r\1_{A_r}$. By the defining property of conditional expectation, for $B\in\{A_l,A_r\}$,
\[
\int_B X\,d\nu_A=\int_B f\,d\nu_A.
\]
Thus $u_l=\frac{1}{\nu_A(A_l)}\int_{A_l}f\,d\nu_A=\frac{1}{\nu(A_l)}\int_{A_l}f\,d\nu=c_{A_l}$
and similarly $u_r=c_{A_r}$. Therefore, $\nu_A$-a.e.,
\[
X(z)=\E_{\nu_A}[f\mid\mathcal G](z)=
\begin{cases}
c_{A_l}, & z\in A_l,\\
c_{A_r}, & z\in A_r.
\end{cases}
\]

Let $w:=\nu_A(A_l)=\nu(A_l)/\nu(A)$. Then
\[
\E_{\nu_A}[\varphi(X)]
= w\,\varphi(c_{A_l})+(1-w)\,\varphi(c_{A_r})
= \frac{\nu(A_l)}{\nu(A)}\varphi(c_{A_l})+\frac{\nu(A_r)}{\nu(A)}\varphi(c_{A_r}).
\]
Also, by the tower property,
\[
\E_{\nu_A}[X]=\E_{\nu_A}\big[\E_{\nu_A}[f\mid\mathcal G]\big]=\E_{\nu_A}[f]=c_A.
\]
Plugging these identities into \eqref{eq:gain_decomp} yields
\[
G(\pi',\pi)
= \nu(A)\Big(\E_{\nu_A}[\varphi(X)]-\varphi(\E_{\nu_A}[X])\Big),
\]
as claimed.
\end{proof}

Since the gain is a Jensen Gap, it is immediate that the gain is always greater than or equal to zero.

\begin{corollary}[Non-negativity of the gain]
\label{cor:nonnegativity_gain}
    For any split $\{A\} \to \{A_l, A_r\}$, if $\nu(A), \nu(A_l), \nu(A_r) > 0$, then $G(\{A_l, A_r\}, \{A\}) \geq 0$.
    
\end{corollary}

\begin{proof}
    Use the non-negativity of $\nu(A)$ and the non-negativity of the Jensen Gap in~\cref{prop:jensen_gap} to see that $G(\pi', \pi) \geq 0$.
\end{proof}

\paragraph{Normalization and empirical log-loss.}
For a partition $\pi$, let $\widehat f_\pi$ be the empirical estimator and
$s_\pi(x):=\int \widehat f_\pi(x,y)\,d\mu_Y(y)$. The normalized empirical
log-loss satisfies
\[
-\frac1N\sum_{i=1}^N\log\frac{\widehat f_\pi(X_i,Y_i)}{s_\pi(X_i)}
=-\frac1N\sum_{i=1}^N\log\widehat f_\pi(X_i,Y_i)
+\frac1N\sum_{i=1}^N\log s_\pi(X_i).
\]
Because the outcome cells cover the truncated outcome domain,
$N^{-1}\sum_i s_\pi(X_i)=N^{-1}\sum_{A\in\pi}n_{XY}(A)=1$; hence Jensen's
inequality gives $N^{-1}\sum_i\log s_\pi(X_i)\le 0$. Thus the unnormalized
empirical loss upper-bounds the normalized one. Since the initial root
partition has $s_\pi(X_i)=1$, the cumulative unnormalized split gains
lower-bound the reduction in normalized empirical log-loss.

\clearpage
\section{Tree Construction and Split Search}\label{app:algorithm_details}

\subsection{Continuous Coordinates}\label{app:algorithm_continuous}

\Cref{alg:split_continuous} illustrates how the split search can be executed for continuous variables. The algorithm does not use presorting (storing the sample order for each feature at the root), which could further improve runtime.

\begin{algorithm}[h]
\small
\caption{Find Best Split (Continuous Coordinate $z$)}
\label{alg:split_continuous}
\begin{algorithmic}[1]
\STATE {\bfseries Input:} Coordinate values $z\in\mathbb R^N$, weights $w\in\mathbb R_+^N$,
node index sets $I_x$ (points with $x_i\in A_X$), $I_{xy}$ (points with $(x_i,y_i)\in A$),
flag \texttt{is\_X\_feature}, routine \texttt{is\_valid}(t), and access to $\mu_Y(A_{l,Y}),\mu_Y(A_{r,Y})$.
\STATE {\bfseries Output:} Best gain $G^\star$ and child index sets.

\STATE $G^\star\leftarrow -\infty$
\STATE $I_{x,\mathrm{sort}}\leftarrow \mathrm{argsort}(z[I_x])$,\quad $I_{xy,\mathrm{sort}}\leftarrow \mathrm{argsort}(z[I_{xy}])$
\STATE $I_x^{\uparrow}\leftarrow I_x[I_{x,\mathrm{sort}}]$,\quad $I_{xy}^{\uparrow}\leftarrow I_{xy}[I_{xy,\mathrm{sort}}]$
\STATE $S_x\leftarrow z[I_x^{\uparrow}]$,\quad $S_{xy}\leftarrow z[I_{xy}^{\uparrow}]$
\STATE $W_x\leftarrow \mathrm{cumsum}(w[I_x^{\uparrow}])$,\quad $W_{xy}\leftarrow \mathrm{cumsum}(w[I_{xy}^{\uparrow}])$
\STATE $N_X\leftarrow W_x[-1]$,\quad $N_{XY}\leftarrow W_{xy}[-1]$

\IF{\texttt{is\_X\_feature}}
    \STATE $\mathcal C \leftarrow$ midpoints between consecutive distinct values in $S_x$
\ELSE
    \STATE $\mathcal C \leftarrow$ midpoints between consecutive distinct values in $S_{xy}$
\ENDIF

\STATE Initialize pointers $p_x\leftarrow 0$,\ $p_{xy}\leftarrow 0$ \COMMENT{upper-bound positions}
\FOR{threshold $t \in \mathcal C$ in increasing order}
    \IF{\texttt{is\_X\_feature}}
        \STATE \COMMENT{Advance $p_x$ to the last index with $S_x[p_x]\le t$}
        \WHILE{$p_x+1<|S_x|$ \AND $S_x[p_x+1]\le t$}
            \STATE $p_x\leftarrow p_x+1$
        \ENDWHILE
        \STATE $n_X(A_l)\leftarrow W_x[p_x]$,\quad $n_X(A_r)\leftarrow N_X-n_X(A_l)$
    \ELSE
        \STATE \COMMENT{Y-split does not change $A_X$}
        \STATE $n_X(A_l)\leftarrow N_X$,\quad $n_X(A_r)\leftarrow N_X$
    \ENDIF

    \STATE \COMMENT{Advance $p_{xy}$ to the last index with $S_{xy}[p_{xy}]\le t$}
    \WHILE{$p_{xy}+1<|S_{xy}|$ \AND $S_{xy}[p_{xy}+1]\le t$}
        \STATE $p_{xy}\leftarrow p_{xy}+1$
    \ENDWHILE
    \STATE $n_{XY}(A_l)\leftarrow W_{xy}[p_{xy}]$,\quad $n_{XY}(A_r)\leftarrow N_{XY}-n_{XY}(A_l)$

    \STATE Compute $G$ from Eq.~\eqref{eq:empirical_gain} using these counts and $\mu_Y(A_{l,Y}),\mu_Y(A_{r,Y})$.
    \IF{\texttt{is\_valid}(t)$ \AND G>G^\star$}
        \STATE $G^\star\leftarrow G$, store best $(t^\star,p_x^\star,p_{xy}^\star)$
    \ENDIF
\ENDFOR

\IF{\texttt{is\_X\_feature}}
    \STATE $I_x^{(l)}\leftarrow I_x^{\uparrow}[0:p_x^\star+1]$,\quad $I_x^{(r)}\leftarrow I_x^{\uparrow}[p_x^\star+1: ]$
\ELSE
    \STATE $I_x^{(l)}\leftarrow I_x$,\quad $I_x^{(r)}\leftarrow I_x$
\ENDIF
\STATE $I_{xy}^{(l)}\leftarrow I_{xy}^{\uparrow}[0:p_{xy}^\star+1]$,\quad $I_{xy}^{(r)}\leftarrow I_{xy}^{\uparrow}[p_{xy}^\star+1: ]$

\STATE {\bfseries Return} $G^\star$, $I_x^{(l)},I_x^{(r)},I_{xy}^{(l)},I_{xy}^{(r)}$
\end{algorithmic}
\end{algorithm}

\paragraph{Example (one continuous feature at one node).}
Suppose a leaf $A$ contains $n_X(A)=6$ covariate points with one continuous coordinate
$z$ (weights $w_i\equiv 1$ for simplicity) and $n_{XY}(A)=4$ joint points. Let the
sorted values inside the node be
\[
S_x = (1,\;2,\;4,\;7,\;9,\;10),\qquad
S_{xy} = (2,\;4,\;9,\;10).
\]
Candidate thresholds are the midpoints between consecutive distinct values in $S_x$:
\[
\mathcal C=\Big\{\tfrac{1+2}{2}=1.5,\ \tfrac{2+4}{2}=3,\ \tfrac{4+7}{2}=5.5,\ \tfrac{7+9}{2}=8,\ \tfrac{9+10}{2}=9.5\Big\}.
\]
Consider $t=5.5$. The split index in $S_x$ is the last position with value $\le 5.5$,
namely after $(1,2,4)$, so
\[
n_X(A_l)=3,\qquad n_X(A_r)=3.
\]
Likewise, the split index in $S_{xy}$ is after $(2,4)$, so
\[
n_{XY}(A_l)=2,\qquad n_{XY}(A_r)=2.
\]
Therefore, the gain for this candidate is computed by plugging these four counts
(and the corresponding $\mu_Y(A_{l,Y}),\mu_Y(A_{r,Y})$) into \eqref{eq:empirical_gain}.
Repeating this for each $t\in\mathcal C$ only requires locating the split index in the sorted arrays and reading prefix sums (or, if scanning thresholds in increasing order, just incrementing the split index), which is exactly why continuous splits can be evaluated efficiently once the values are sorted.

\subsection{Efficient subset splitting for categorical covariates}
\label{app:cat_subset_split}

We now show how to efficiently find the best split for a categorical coordinate. The approach is similar to that used in standard decision trees, as shown by \citet{fisher1958grouping}.

\paragraph{Problem setting (one categorical $X$-coordinate at one leaf).}
Fix a current leaf $A=A_X\times A_Y$ and a categorical covariate coordinate
$z_\ell\in\Sigma=\{1,\dots,K\}$.
We consider subset splits $S\subset\Sigma$, $S\neq\emptyset,\Sigma$:
\[
A_l := A\cap\{z_\ell\in S\},\qquad A_r := A\cap\{z_\ell\notin S\}.
\]
Since this is an $X$-split, both children share the same outcome side:
$A_{l,Y}=A_{r,Y}=A_Y$, hence $\mu_Y(A_{l,Y})=\mu_Y(A_{r,Y})=\mu_Y(A_Y)$ is constant
and cancels from the empirical gain \eqref{eq:empirical_gain}.

For each category $c\in\Sigma$, define the per-category counts inside $A$:
\[
a_c := n_{XY}\!\big(A\cap\{z_\ell=c\}\big),\qquad
b_c := n_{X}\!\big(A\cap\{z_\ell=c\}\big),
\qquad r_c := \frac{a_c}{b_c},
\]
where we only consider categories with $b_c>0$.
For $S\subset\Sigma$ set
\[
a_S:=\sum_{c\in S}a_c,\quad b_S:=\sum_{c\in S}b_c,\quad r_S:=\frac{a_S}{b_S},
\]
and similarly for $S^c$.

For an $X$-split, $\mu_Y(A_{l,Y})=\mu_Y(A_{r,Y})=\mu_Y(A_Y)$ is constant, and the parent term in \eqref{eq:empirical_gain} is constant in $S$. Dropping constants yields
\begin{equation}
\label{eq:gain_S}
G(S)\equiv a_S\log\!\Big(\frac{a_S}{b_S}\Big)+a_{S^c}\log\!\Big(\frac{a_{S^c}}{b_{S^c}}\Big)
= b_S\varphi(r_S)+b_{S^c}\varphi(r_{S^c}).
\end{equation}

which we seek to maximize.

\begin{proposition}[Optimal subset split is a sorted prefix]
\label{prop:cat_prefix_bregman}
Let $\varphi(u):=u\log u$ and let $D(u,v)$ be its Bregman divergence
\[
D(u,v):=\varphi(u)-\varphi(v)-\varphi'(v)(u-v)=u\log\!\frac{u}{v}-u+v.
\]

Consequently, there exists a maximizer $S^\star$ of $G(S)$ such that, after sorting categories by $r_c$,
\[
r_{\sigma(1)}\le \cdots \le r_{\sigma(K)},
\qquad
S^\star=\{\sigma(1),\dots,\sigma(t)\}\ \text{for some }t\in\{1,\dots,K-1\}.
\]
In particular, the best subset split can be found by sorting $\{r_c\}$ once and scanning the $K-1$ prefix thresholds using prefix sums of $\{a_c,b_c\}$.
\end{proposition}

\begin{proof}

Using $\sum_{c\in S} b_c(r_c-r_S)=0$ and the definition of $D$, we have
\[
\sum_{c\in S} b_c D(r_c,r_S)=\sum_{c\in S} b_c\varphi(r_c)-b_S\varphi(r_S),
\]
and similarly for $S^c$. By rearranging and plugging into \eqref{eq:gain_S}, we can see that the empirical gain of the categorical $X$-split $S$ equals (up to additive constants independent of $S$):
\[
G(S) \ \equiv\ b_S\varphi(r_S)+b_{S^c}\varphi(r_{S^c})
\ =\ \sum_{c\in\Sigma} b_c\varphi(r_c)\;-\;
\Big(\sum_{c\in S} b_c D(r_c,r_S)+\sum_{c\in S^c} b_c D(r_c,r_{S^c})\Big).
\]

Let $u:=r_S$ and $v:=r_{S^c}$ and relabel children so that $u\le v$.
Consider the difference in divergence to the two centers for a point $x$:
\[
\Delta(x):=D(x,u)-D(x,v).
\]
A direct expansion shows $\Delta(x)=\alpha+\beta x$ with
$\beta=\varphi'(v)-\varphi'(u)\ge 0$ (since $\varphi'$ is increasing and $u\le v$),
so $\Delta(x)$ is non-decreasing in $x$. Hence, there exists a threshold $t$ such that
categories with $r_c\le t$ weakly prefer $u$ and those with $r_c>t$ weakly prefer $v$.
Therefore, for fixed $(u,v)$, the assignment minimizing:
\[
\sum_c b_c \min\{D(r_c,u),D(r_c,v)\}
\]
is contiguous after sorting by $r_c$,
i.e., a prefix split. Finally, for Bregman divergences, the minimizer over the center of $\sum_{c\in S} b_c D(r_c,\cdot)$ is the weighted mean $r_S$; thus, re-centering after applying the threshold assignment cannot increase the cost. This yields an optimal split of prefix form.
\end{proof}

\paragraph{Remark (categorical $Y$-splits).}
If the split acts on a categorical outcome coordinate $z_\ell\in\Sigma$ (a $Y$-split), then $n_X(\cdot)$ is constant across children ($n_X(A_l)=n_X(A_r)=n_X(A)$), but the $\mu_Y$ terms vary across subsets. In this case, the same sorted-prefix conclusion holds after redefining, for each category $c\in\Sigma$,
\[
a_c := n_{XY}\!\big(A\cap\{z_\ell=c\}\big),\qquad
b_c := \mu_Y\!\big(A_Y\cap\{z_\ell=c\}\big),\qquad
r_c := \frac{a_c}{b_c},
\]
and repeating the Bregman-threshold argument with these $(a_c,b_c,r_c)$.

\paragraph{Time complexity consequences.}
The prefix optimality avoids enumerating all $2^{|\Sigma|}-2$ nontrivial subset splits.
Instead, one computes per-category statistics $(a_c,b_c)$ inside the leaf and sorts
categories by $r_c=a_c/b_c$, then scans the $|\Sigma|-1$ prefixes using prefix sums.
This costs $O(n_{XY}(A)+|\Sigma|\log|\Sigma|)$ time (or $O(|\Sigma|\log|\Sigma|)$ once
the per-category counts are available).

\subsection{Best-first growth under a split budget}\label{app:algorithm_growth}

The practical tree-growing rule is gain-based best-first search under a global split budget $k_N$. For each current leaf $A$, we evaluate admissible one-coordinate splits, record the best empirical log-loss gain, and split the leaf--split pair with the largest gain among all leaves. Operationally, we maintain a priority queue keyed by each leaf's current best gain; after a split, only the affected entries need to be recomputed.


\clearpage
\section{Consistency Proofs}\label{app:consistency}

To prove consistency, we use the results from \cite{lugosi1996consistency} for consistency of data-driven partitions.
We rephrase Proposition 1 of their work for clarity.

\begin{corollary}
\label{cor:lugosi1}
    \cite{lugosi1996consistency} Let $Z_1, Z_2, \dots$ be i.i.d. random vectors in $\X \times \Y$ with $Z_i \sim \lambda$, and let $\mathcal A_1, \mathcal A_2, \dots$ be a sequence of partition families. If $N$ tends to infinity:

    \begin{enumerate}
        \item $N^{-1}m(\mathcal{A}_N) \rightarrow 0$
        \item $N^{-1}\log\Delta_N^*(\mathcal{A}_N) \rightarrow 0$
    \end{enumerate}

    Then 
    \begin{equation}
        \sup_{\pi \in \mathcal{A}_N}\sum_{A\in\pi}|\lambda_N(A) - \lambda (A)| \rightarrow 0
    \end{equation}
    with probability one.
\end{corollary}
    
We now adapt the problem setting of \cite{lugosi1996consistency} to our setting.
A $N$-sample partitioning rule $\pi_N$ associates every dataset
$\D_N\in\Z^N$ with a measurable partition of the truncated domain
$\bZN(\D_N)$. Let $\Pi=\{\pi_1,\pi_2,\dots\}$ be a partitioning scheme, and let
$\A_N:=\{\pi_N(\D_N):\D_N\in\Z^N\}$ be the corresponding non-random family of
truncated partitions.

When applying uniform convergence results on $\Z$, we use the full-space
extension
\[
\widetilde{\pi}_N(\D_N)
:=\pi_N(\D_N)\cup\{\Z\setminus\bZN(\D_N)\},
\]
omitting empty cells, and denote the corresponding family by
$\widetilde{\A}_N$. These extensions add at most one cell, so
$m(\widetilde{\A}_N)\le m(\A_N)+1$. Moreover, for any finite
$\D\subseteq\Z$, the induced full-space partition is determined by the
truncated-cell intersections and by which points of $\D$ fall outside
$\bZN(\D_N)$. The latter membership pattern belongs to the class
$\mathcal B:=\{\X\times B:\ B\subseteq\Y \text{ is an axis-aligned box}\}$,
which has finite VC dimension in the mixed-type setup. Hence, by
Sauer--Shelah,
\[
\Delta_N^*(\widetilde{\A}_N)
\le
\Delta_N^*(\A_N)\,\Delta_{\mathcal B}(N),
\qquad
\frac{1}{N}\log\Delta_{\mathcal B}(N)\to 0.
\]
Therefore conditions 1 and 2 in~\cref{the:consistency_1} also hold for
$\widetilde{\A}_N$.

The mixed-type setup and truncation construction are collected in~\Cref{app:estimator_details}, and the practical best-first split-search mechanics are collected in~\Cref{app:algorithm_details}. This appendix records the auxiliary results used in the proof of~\Cref{the:consistency_1}.

\subsection{Uniform bound for VC Dimension of X Partitions}

In the main proof regarding consistency of data-driven partitions of $\Z$, the bound proven in the Lemma \ref{lem:vc_dev_PX_as_only} will be used.

\begin{lemma}[Almost-sure uniform VC deviation for $\hPX$]
\label{lem:vc_dev_PX_as_only}
Let $\mathcal C_X$ have VC dimension $V_X<\infty$ and an indicator class
with a countable pointwise-dense subclass, as for the coordinate boxes used
here. Then with probability one, for all sufficiently large $N$,
\[
\sup_{U\in\mathcal C_X}|\hPX(U)-\PP_X(U)|
\ \le\
\sqrt{\frac{8}{N}\Big(V_X\log(2eN/V_X)+\log(4N^2)\Big)}
\ =\ O\Big(\sqrt{\tfrac{\log N}{N}}\Big).
\]
When $V_X=0$, interpret $V_X\log(2eN/V_X)$ as zero.
\end{lemma}

\begin{proof}
Let $\mathcal F:=\{\mathbf 1_U:\ U\in\mathcal C_X\}$. For $f\in\mathcal F$, write
$\PP f:=\E[f(X)]$ and $\widehat{\PP} f:=\frac1N\sum_{i=1}^N f(X_i)$.

\paragraph{Step 1: Symmetrization.}
Let $X_1',\dots,X_N'$ be an independent ghost sample with law $\PP_X$, and let
$\widehat{\PP}' f:=\frac1N\sum_{i=1}^N f(X_i')$. For $N\varepsilon^2\geq2$,
\begin{equation}
\label{eq:vc_symm_merged}
\PP\Big(\sup_{f\in\mathcal F}|\widehat{\PP} f-\PP f|>\varepsilon\Big)
\ \le\
2\,\PP\Big(\sup_{f\in\mathcal F}|\widehat{\PP} f-\widehat{\PP}' f|>\varepsilon/2\Big).
\end{equation}

\paragraph{Step 2: Finite reduction and Hoeffding.}
Introduce independent signs $\sigma_1,\ldots,\sigma_N$, each uniform on
$\{-1,1\}$ and independent of both samples. Swapping $X_i$ and $X_i'$ when
$\sigma_i=-1$ leaves the joint sample law unchanged. Hence the supremum of
$|\widehat{\PP}f-\widehat{\PP}'f|$ has the same distribution as that of
$|N^{-1}\sum_i\sigma_i(f(X_i)-f(X_i'))|$.
Condition on the realized $2N$ points $(X_1,\dots,X_N,X_1',\dots,X_N')$.
The supremum over $f\in\mathcal F$ depends only on the induced labeling of these $2N$
points, hence is a maximum over at most $\Delta_{\mathcal C_X}(2N)$ labelings.
For each fixed labeling, the signed differences are conditionally independent,
mean zero, and in $[-1,1]$, so Hoeffding gives
\[
\PP_{\sigma}\Big(
\Big|\frac1N\sum_{i=1}^N\sigma_i(f(X_i)-f(X_i'))\Big|>\varepsilon/2
\,\Big|\,X_1,\ldots,X_N,X_1',\ldots,X_N'\Big)
\le 2\exp\Big(-\frac{N\varepsilon^2}{8}\Big).
\]
A conditional union bound, averaging over both samples, and then
\eqref{eq:vc_symm_merged} yield
\begin{equation}
\label{eq:vc_tail_merged}
\PP\Big(\sup_{U\in\mathcal C_X}|\hPX(U)-\PP_X(U)|>\varepsilon\Big)
\ \le\
4\,\Delta_{\mathcal C_X}(2N)\exp\Big(-\frac{N\varepsilon^2}{8}\Big).
\end{equation}

\paragraph{Step 3: Sauer--Shelah.}
Since $\mathrm{VCdim}(\mathcal C_X)=V_X$, Sauer--Shelah implies for $N\ge V_X$,
\[
\Delta_{\mathcal C_X}(2N)\ \le\ \Big(\frac{2eN}{V_X}\Big)^{V_X}.
\]
For $V_X=0$, the growth function and the right-hand side are both one.
Substituting into \eqref{eq:vc_tail_merged} gives, for $N\ge V_X$ and
$N\varepsilon^2\geq2$,
\begin{equation}
\label{eq:vc_tail_sauer_merged}
\PP\Big(\sup_{U\in\mathcal C_X}|\hPX(U)-\PP_X(U)|>\varepsilon\Big)
\ \le\
4\Big(\frac{2eN}{V_X}\Big)^{V_X}\exp\Big(-\frac{N\varepsilon^2}{8}\Big).
\end{equation}

\paragraph{Part (1): High-probability bound.}
For $0<\delta<1$, choose $\varepsilon=\varepsilon_N(\delta)$ so that the right-hand side of
\eqref{eq:vc_tail_sauer_merged} is at most $\delta$:
\[
\varepsilon_N(\delta)
:= \sqrt{\frac{8}{N}\Big(V_X\log(2eN/V_X)+\log(4/\delta)\Big)}.
\]
This choice satisfies $N\varepsilon_N(\delta)^2\geq2$. Thus, for $N\geq V_X$,
\[
\PP\Big(\sup_{U\in\mathcal C_X}|\hPX(U)-\PP_X(U)|>\varepsilon_N(\delta)\Big)
\le \delta,
\]

\paragraph{Part (2): Almost-sure bound (Borel--Cantelli).}
Set $\delta_N:=N^{-2}$ and apply (1) for $N\geq\max\{2,V_X\}$:
\[
\PP\Big(\sup_{U\in\mathcal C_X}|\hPX(U)-\PP_X(U)|
>\varepsilon_N(\delta_N)\Big)\ \le\ \delta_N.
\]
Since $\sum_{N\ge 1}\delta_N<\infty$, Borel--Cantelli implies that with probability one,
only finitely many of these events occur; equivalently, almost surely for all large $N$,
\[
\sup_{U\in\mathcal C_X}|\hPX(U)-\PP_X(U)|\ \le\ \varepsilon_N(\delta_N).
\]
Finally,
\[
\varepsilon_N(\delta_N)
=\sqrt{\frac{8}{N}\Big(V_X\log(2eN/V_X)+\log(4N^2)\Big)}
=: \varepsilon_N^{\mathrm{a.s.}} = O\left(\sqrt{\log N / N}\right)
\]
for all sufficiently large $N$.
\end{proof}

\newcommand{\termone}{\|\fn - \fxn\|_{L^1(\nu;\bZN)}}
\newcommand{\termtwo}{\|\fxn - \fs\|_{L^1(\nu;\bZN)}}
\newcommand{\termthree}{\|\fs - f\|_{L^1(\nu;\bZN)}}

\subsection{Proof of Theorem \ref{the:consistency_1}}\label{app:fullproofthm31}

We now state the proof of the consistency theorem of data-driven partitions.
For mixed spaces, we use the product metric associated with the compactified
continuous coordinates:
\[
\dham(u,v):=\mathbf 1\{u\neq v\}.
\]
For $z=(x,y)$ and $z'=(x',y')$, define
\[
d_\Z(z,z')
:=\dcont(z,z')
   + \sum_{j=1}^{\dxk}\dham(x^{(k)}_j,x'^{(k)}_j)
   + \sum_{j=1}^{\dyk}\dham(y^{(k)}_j,y'^{(k)}_j).
\]

Then $\distZ(z,S):=\inf_{u\in S} d_{\Z}(z,u)$ for $S\subset\Z$. 

\ConsistencyOne*

\begin{proof}[Proof of Theorem \ref{the:consistency_1}]

Let $\bar{\Y}_N$ be defined as in \eqref{eq:yn_improved} and set $\bZN:=\X\times\bar{\Y}_N$. We can write:

\[
\|f_N-f\|_{L^1(\nu)}
=\int_{\X}\int_{\bYN}\!|f_N-f|\,d\mu_Y\,d\PP_X
+\int_{\X}\int_{\bYN^{\,c}}\!|f_N-f|\,d\mu_Y\,d\PP_X.
\]

Since $f_N = 0$ on $\bZ^C$, the last term equals $\PY(\bYN^C)$, which goes to zero almost surely as $N\to\infty$ according to Lemma \ref{lem:box_mass}. Therefore, we can focus on $\|\fn-f\|_{L^1(\nu;\bar{\Z}_N)}$.

Let:
\[
f_N'(z) = \frac{\hPXY(\pi_N[z])}{\PX(\pi_N[z])\mu_Y(\pi_N[z])}, \quad \fstar_N = \frac{\PXY(\pi_N[z])}{\PX(\pi_N[z])\mu_Y(\pi_N[z])}
\]
By triangle inequality:

\[
||\fn - f||_{\lnuz} \leq \underbrace{\termone}_{(\mathrm{I})}
+
\underbrace{\termtwo}_{(\mathrm{II})}
+
\underbrace{\termthree}_{(\mathrm{III})}.
\]

\paragraph{Bounding $(I)$}

By~\Cref{ass:admissible_leaves}, almost surely every leaf $A\in\pi_N$ has
$\hPX(A_X)>0$ and $\mu_Y(A_Y)>0$, so the ratios below are well-defined.

\begin{align*}
    \termone &= \|\frac{\hPXY}{\hPX\mu_Y} - \frac{\hPXY}{\PX\mu_Y} \|_{\lnuz} = \| \frac{\hPXY}{\PX\mu_Y}\left(\frac{\PX -\hPX}{\hPX}\right) \|_{\lnuz} \\
    &= \sum\limits_{A \in \pi_N}|\frac{\hPXY(A)}{\PX(A)\mu_Y(A)}\left(\frac{\PX(A) -\hPX(A)}{\hPX(A)}\right)| \PX(A)\mu_Y(A) \\
    &= \sum\limits_{A \in \pi_N} |\frac{\hPXY(A)}{\hPX(A)} ( \PX(A) -\hPX(A))|
\end{align*}
 
The key point is that the $A_X$'s may overlap, so we \emph{cannot} invoke a
partition-based bound on the sum. Instead, we group terms by \emph{distinct}
projected sets. Let $\mathcal C_X$ denote the collection of all covariate
regions in $\X$ that can arise as the $X$-projection of a leaf by iterating the
permitted split tests (axis-aligned threshold splits for continuous features
and subset-membership splits for categorical features). Assume that
$\mathcal C_X$ has finite VC dimension, and denote
$V_X := \mathrm{VCdim}(\mathcal C_X)$.

Let
\[
\mathcal U(\pi_N):=\{A_X:\ A\in\pi_N\}\subseteq\mathcal C_X,
\]
and for each $U\in\mathcal U(\pi_N)$ define
\[
w_U
:=\sum_{A\in\pi_N:\ A_X=U}\frac{\widehat{\PP}_{XY}(A)}{\hPX(U)}.
\]
Because $\bigcup_{A:\,A_X=U}A \subseteq U\times\Y$ and empirical mass is additive
over disjoint leaves,
\[
\sum_{A:\,A_X=U}\widehat{\PP}_{XY}(A)\ \le\ \widehat{\PP}_{XY}(U\times\YN)=\hPX(U),
\]
hence $0\le w_U\le 1$. Therefore
\begin{align*}
\|\fn-\fxn\|_{L^1(\nu;\bZN)}
&=
\sum_{U\in\mathcal U(\pi_N)} w_U\,\big|\PP_X(U)-\hPX(U)\big|
\ \le\
\sum_{U\in\mathcal U(\pi_N)} \big|\PP_X(U)-\hPX(U)\big| \\
&\le
|\mathcal U(\pi_N)|\cdot \sup_{U\in\mathcal C_X}\big|\PP_X(U)-\hPX(U)\big|.
\end{align*}
Now apply Lemma~\ref{lem:vc_dev_PX_as_only}: $\sup_{U\in\mathcal C_X}\big|\PP_X(U)-\hPX(U)\big|$ is almost surely bounded with rate $O(\sqrt{\log N / N})$. Also, $\mathcal U(\pi_N) \leq m_X(\A_N)$. Consequently, if
\[
m_X(\mathcal A_N)\,\sqrt{\frac{\log N}{N}}\ \longrightarrow\ 0,
\]
then by condition 4. $(\mathrm{I})=\|\fn-\fxn\|_{L^1(\nu;\bZN)}\to 0$ almost surely.

\paragraph{Bounding $(II)$}

\begin{align*}
    \termtwo = \|\frac{1}{\PX\mu_Y}(\hPXY - \PXY)\|_{\lnuz} = \sum_{A\in\pi_N}|\hPXY(A) - \PXY(A)|
\end{align*}

For the application of~\cref{cor:lugosi1}, view the truncated partition as a
full-space partition by adding the complement of the random truncation:
\[
\widetilde{\pi}_N:=\pi_N\cup\{\Z\setminus\bZN\},
\]
omitting empty cells. The family $\A_N$ in~\cref{the:consistency_1} is this
truncated partition family, while $\widetilde{\A}_N$ denotes the corresponding
full-space extension family defined above. Since
$\pi_N\subseteq\widetilde{\pi}_N$,
\[
\sum_{A\in\pi_N}|\hPXY(A)-\PXY(A)|
\le
\sum_{A\in\widetilde{\pi}_N}|\hPXY(A)-\PXY(A)|.
\]
Applying~\cref{cor:lugosi1} to the full-space family $\widetilde{\A}_N$ with
$\lambda=\PXY$, and using the inherited versions of conditions 1 and 2, this
term also converges to 0.

\paragraph{Bounding $(III)$}To bound the term $(III)$, we use a similar strategy to \cite{lugosi1996consistency}, but in a more general setting where $\X\times\Y$ can be composed of integer-valued, discrete, or categorical dimensions.

Since $(\mathcal Z,d_{\mathcal Z})$ is a separable metric space and $\nu$ is a Borel
measure, for every $\varepsilon>0$ there exists a nonnegative simple function
$g=\sum_{j=1}^J a_j \mathbf 1_{G_j}$ with pairwise disjoint $G_j\subseteq\mathcal Z$
and $\nu(\partial G_j)=0$ such that
\[
\|f-g\|_{L^1(\nu)}<\varepsilon.
\]
In particular,
\[
\|f-g\|_{L^1(\nu;\bar{\mathcal Z}_N)} \le \|f-g\|_{L^1(\nu)}<\varepsilon.
\]
Define
\[
g_N(z):=\frac{\int_{\pi_N[z]} g\,d\nu}{\nu(\pi_N[z])}.
\]

We split $(III)$ into three components, based on these newly introduced functions $g$ and $g_N$:
\[
\termthree \leq \underbrace{\| \fs - g_N\|_{\lnuz}}_{(*)} + \underbrace{\|g_N - g\|_{\lnuz}}_{(\dagger)} + \underbrace{\|g - f\|_{\lnuz}}_{<\varepsilon}
\]

The third term is bounded by definition. Bounding $(*)$:

\begin{align}
    \| \fs - g_N\|_{\lnuz} &= \int_{\bZN} |\fs - g_N|\,d\nu = \sum_{A \in \pi_N}|\int_Af\, d\nu - \int_Ag\,d \nu| \\
    &\leq \sum_{A \in \pi_N}\int_A|f -g|\,d\nu = \int_{\bZN}|f-g|\,d\nu <\varepsilon
\end{align}

\newcommand{\dist}{\text{dist}}
\newcommand{\pG}{\partial G}
To bound $(\dagger)$, let $M:= \|g\|_\infty = \max\limits_{j\in\{1,\dots,J\}}a_j$. Let $\partial G_j$ denote the topological boundary of $G_j$, $\nu(\partial G_j) = 0$ for all $j\in \{1,\dots,J\}$. Choose $ 1 > \eta > 0$ and define two sets:
\begin{equation*}
U(\eta)
:=\bigcup_{j=1}^J \{z\in\bar{\Z}_N:\ \distZ(z,\partial G_j)\le \eta\}
\end{equation*}
\begin{equation*}
D_N(\eta) = D_N^{cont}(\eta) \cup D_N^{cat}
\end{equation*}

where $D_N^{\mathrm{cont}}(\eta)=\{z\in\bar{\mathcal Z}_N:\ \diamc(\pi_N[z])>\eta\}$
and $D_N^{\mathrm{cat}}=\{z:\ \exists k\ \text{s.t.}\ |S_k(\pi_N[z])|>1\}$.

We study $(\dagger)$ in two disjoint domains: $\bZ \setminus (U(\eta) \cup D_N(\eta))$ and $U(\eta) \cup D_N(\eta)$.

\begin{claim}
    For all, $z \in \bZ \setminus (U(\eta) \cup D_N(\eta))$ we have $g_N=g$ and  $\int_{\bZ \setminus (U(\eta) \cup D_N(\eta))}|g_N - g|d\nu = 0$.
\end{claim}
\begin{proof}
    Pick $z \in \bZ \setminus  U(\eta)$. Then there exists a $G_j$ such that $z\in G_j$ and $\distZ(z, \pG_j) > \eta$. Let $B(z, \eta) \subseteq G_j$ denote a ball of radius $\eta$ centered on $z$. If, in addition, $z\notin D_N(\eta)$ then (i) $|S_k(\pi_N[z])|=1$ for every categorical coordinate $k$, so $d_\Z(z,u)=d_{\mathrm{cont}}(z,u)$ for all $u\in\pi_N[z]$, and (ii) $\diamc(\pi_N[z])\le \eta$, hence $d_{\mathrm{cont}}(z,u)\le \eta$ for all $u\in\pi_N[z]$.
    
Therefore $d_\Z(z,u)\le \eta$ for all $u\in\pi_N[z]$, i.e.\ $\pi_N[z]\subseteq B(z,\eta)$.
Hence $\pi_N[z] \subseteq B(z, \eta) \subseteq G_j$. This implies $g = a_j$ on $\pi_N[z]$ and $g_N = g$.
\end{proof}
Now we focus on $z \in U(\eta) \cup D_N(\eta)$. With the claim above, we have:

\begin{align}
\label{eq:gn_g}
\|g_N - g\|_{\lnuz} &\leq \int_{U(\eta) \cup D_N(\eta)}|g_N - g| d\nu  \leq \int_{D_N(\eta)}|g_N - g| d\nu + \int_{U(\eta)}|g_N - g| d\nu 
\end{align}

Since $\partial G_j$ is closed and $\nu(\partial G_j)=0$, we have
\[
\{z:\distZ(z,\partial G_j)\le \eta\}\downarrow \partial G_j \quad (\eta\downarrow 0),
\]
hence by continuity from above for finite measures,
$\nu(\{z:\distZ(z,\partial G_j)\le \eta\})\to \nu(\partial G_j)=0$.
Therefore $\nu(U(\eta))\to 0$ as $\eta\downarrow 0$. Since $|g| \leq M$ and $|g_N| \leq M$, the second term of \eqref{eq:gn_g} goes to zero as $\eta \to 0$:
\[
\int_{U(\eta)}|g_N - g| d\nu \leq 2M\int_{U(\eta)}d\nu = 2M\nu(U(\eta)) \xrightarrow{\eta \to 0}0
\]

Let $\pi_N^\eta := \{A\in\pi_N:\ \diamc(A)>\eta\ \text{ or }\ \exists k\ \text{s.t.}\ |S_k(A)|>1\}$. With respect to the first term of \eqref{eq:gn_g}:

\begin{align}
    \int_{D_N(\eta)}|g_N - g| d\nu
    &\leq \int_{D_N(\eta)} g_N\,d\nu+\int_{D_N(\eta)}g\,d\nu
    = 2\int_{D_N(\eta)}g\,d\nu \\
    &=2\int_{D_N(\eta)}|g - f + f| d\nu \\
    &\leq  2\PXY(D_N(\eta)) + 2\int_{D_N(\eta)}|g - f| d\nu \\
    &< 2\PXY(D_N(\eta)) + 2\varepsilon
\end{align}

We now show how condition 3 implies that the first term goes to zero as $N$ tends to infinity. Fix any $\eta > 0$. Let $c_{min} = \min_k \frac{1}{|\Sigma_k| - 1}$  denote the minimum length of a cell of $\pi_N$ along a categorical dimension. Let $\gamma_N < \min(\eta, c_{min})$. Then:

\begin{enumerate}
    \item If $z \in D_N^{cont}(\eta)$, then $\diamh(\pi_N[z]) \geq \diamc(\pi_N[z]) > \eta >\gamma_N$
    \item If $z \in D_N^{cat}$, then $\diamh(\pi_N[z]) \geq c_{min} > \gamma_N$
\end{enumerate}

For all large $N$:

\[
D_N(\eta) \subseteq \{z \in \bZN: \diamh(\pi_N[z]) > \gamma_N\}
\]

so $\PXY(D_N(\eta)) < \PXY(\{z : \diamh(\pi_N[z]) > \gamma_N\}) \xrightarrow{N \to \infty}{0}$. This finishes the proof for  $(\dagger)$.

Finally, we have, as $N \to \infty$:

\begin{equation}
    ||\fn - f||_{\lnuz} < 4\varepsilon
\end{equation}

Since $\epsilon$ is arbitrary, this proves the convergence for $\dagger$ and therefore the proof is concluded.

\end{proof}

\subsection{Normalization preserves \texorpdfstring{$L^1(\nu)$}{L1(nu)}-consistency}
\label{app:norm_consistency}

In practice, we may normalize a nonnegative estimator $h_N(x,\cdot)$ to integrate to one
with respect to $\mu_Y$. The following lemma shows that this operation does not affect
$L^1(\nu)$ consistency, where $\nu=\PX\otimes\mu_Y$.

\begin{lemma}[Normalization preserves $L^1(\nu)$-consistency]
\label{lem:normalization_consistency}
Let $\mu_Y$ be a $\sigma$-finite measure on $\Y$ and let $\nu:=\PX\otimes\mu_Y$.
Assume $\PXY\ll \nu$ and let
\[
f := \frac{d\PXY}{d(\PX\otimes\mu_Y)}.
\]
Let $h_N:\X\times\Y\to[0,\infty)$ be measurable and define
\[
s_N(x):=\int_\Y h_N(x,y)\,\mu_Y(dy),\qquad
\bar h_N(x,y):=
\begin{cases}
h_N(x,y)/s_N(x), & s_N(x)>0,\\
0, & s_N(x)=0.
\end{cases}
\]
If $\|h_N-f\|_{L^1(\nu)}\to 0$, then $\|\bar h_N-f\|_{L^1(\nu)}\to 0$. Moreover, for all $N$,
\[
\|\bar h_N-f\|_{L^1(\nu)} \;\le\; 6\,\|h_N-f\|_{L^1(\nu)}.
\]
\end{lemma}

\begin{proof}
Define the pointwise $L^1(\mu_Y)$ error
\[
\xi_N(x) := \int_\Y |h_N(x,y)-f(x,y)|\,\mu_Y(dy),
\]
so that $\E_{\PX}[\xi_N]=\|h_N-f\|_{L^1(\nu)}$ by Tonelli's theorem. Next define the
``good'' set
\[
E_N := \{x\in\X:\ \xi_N(x)\le 1/2\}.
\]
Since $\int_\Y f(x,y)\,\mu_Y(dy)=1$ for $\PX$-a.e.\ $x$, we have for such $x$,
\[
|1-s_N(x)|
=
\left|\int_\Y \big(f(x,y)-h_N(x,y)\big)\,\mu_Y(dy)\right|
\le
\int_\Y |f(x,y)-h_N(x,y)|\,\mu_Y(dy)
=
\xi_N(x).
\]
Hence, on $E_N$ we have $s_N(x)\ge 1-|1-s_N(x)|\ge 1-\xi_N(x)\ge 1/2$, so $\bar h_N$ is well-defined there.

We decompose the $L^1(\nu)$ error into contributions over $E_N$ and its complement:
\begin{equation}
\label{eq:norm_split}
\|\bar h_N-f\|_{L^1(\nu)}
=
\int_{E_N}\int_\Y |\bar h_N-f|\,\mu_Y(dy)\,\PX(dx)
+
\int_{E_N^c}\int_\Y |\bar h_N-f|\,\mu_Y(dy)\,\PX(dx).
\end{equation}

\paragraph{Step 1: bound on $E_N$.}
Fix $x\in E_N$. Using the triangle inequality,
\[
\begin{aligned}
\int_\Y |\bar h_N-f|\,\mu_Y(dy)
&\le
\int_\Y |\bar h_N-h_N|\,\mu_Y(dy)
+
\int_\Y |h_N-f|\,\mu_Y(dy) \\
&=
\int_\Y |\bar h_N-h_N|\,\mu_Y(dy) + \xi_N(x).
\end{aligned}
\]
Moreover, for $x\in E_N$ (so $s_N(x)>0$),
\[
\begin{aligned}
\int_\Y |\bar h_N-h_N|\,\mu_Y(dy)
&=
\int_\Y h_N(x,y)\Big|\frac{1}{s_N(x)}-1\Big|\,\mu_Y(dy) \\
&=
\Big|\frac{1}{s_N(x)}-1\Big|\int_\Y h_N(x,y)\,\mu_Y(dy)
=
|1-s_N(x)|.
\end{aligned}
\]
Combining with $|1-s_N(x)|\le \xi_N(x)$ yields
\[
\int_\Y |\bar h_N-f|\,\mu_Y(dy)\le 2\,\xi_N(x)
\qquad\text{for all }x\in E_N.
\]
Integrating over $E_N$ gives
\begin{equation}
\label{eq:good_bound}
\int_{E_N}\int_\Y |\bar h_N-f|\,\mu_Y(dy)\,\PX(dx)
\le
2\int_{E_N}\xi_N(x)\,\PX(dx)
\le
2\,\|h_N-f\|_{L^1(\nu)}.
\end{equation}

\paragraph{Step 2: bound on $E_N^c$.}
For any $x$, both $f(x,\cdot)$ and $\bar h_N(x,\cdot)$ are nonnegative and satisfy
\[
\int_\Y f(x,y)\,\mu_Y(dy)=1
\quad\text{and}\quad
\int_\Y \bar h_N(x,y)\,\mu_Y(dy)\le 1,
\]
where the second inequality follows from the definition of $\bar h_N$ (it equals $1$
when $s_N(x)>0$ and equals $0$ otherwise). Therefore, for all $x$,
\[
\int_\Y |\bar h_N-f|\,\mu_Y(dy)
\le
\int_\Y \bar h_N\,\mu_Y(dy)+\int_\Y f\,\mu_Y(dy)
\le 2.
\]
Hence
\begin{equation}
\label{eq:bad_bound}
\int_{E_N^c}\int_\Y |\bar h_N-f|\,\mu_Y(dy)\,\PX(dx)
\le
2\,\PX(E_N^c).
\end{equation}

It remains to bound $\PX(E_N^c)$. By Markov's inequality,
\[
\PX(E_N^c)=\PX\big(\xi_N>1/2\big)
\le
2\,\E_{\PX}[\xi_N]
=
2\,\|h_N-f\|_{L^1(\nu)}.
\]
Substituting into \eqref{eq:bad_bound} gives
\begin{equation}
\label{eq:bad_bound2}
\int_{E_N^c}\int_\Y |\bar h_N-f|\,\mu_Y(dy)\,\PX(dx)
\le
4\,\|h_N-f\|_{L^1(\nu)}.
\end{equation}

\paragraph{Step 3: conclude.}
Combining \eqref{eq:norm_split}, \eqref{eq:good_bound}, and \eqref{eq:bad_bound2}
yields
\[
\|\bar h_N-f\|_{L^1(\nu)}
\le
6\,\|h_N-f\|_{L^1(\nu)}.
\]
In particular, if $\|h_N-f\|_{L^1(\nu)}\to 0$, then $\|\bar h_N-f\|_{L^1(\nu)}\to 0$.
\end{proof}

\begin{corollary}
\label{cor:extrapolation_consistency}
Consider the setting in~\cref{the:consistency_1} and the definition of $\fpartn$ as in~\cref{eq:density_normalized}. Then
\[
\|\fpartn - f\|_{L^1(\nu)} \longrightarrow 0
\quad\text{as }N\to\infty,
\]
almost surely.
\end{corollary}

\begin{proof}
    Use~\Cref{the:consistency_1,lem:normalization_consistency}.
\end{proof}

\subsection{Discussion of Assumptions in Theorem~\ref{the:consistency_1}}
\label{app:discussion_assumptions}

The assumptions in Theorem~\ref{the:consistency_1} do not directly apply to
Partition Tree without explicit assumptions on the underlying distribution or
constraints on the partitioning strategy. Conditions 1, 2, and 4 are on the
\textit{estimation error} (variance) side and can be controlled by choosing
appropriate hyperparameters, for example, the maximum number of leaves.
We provide a sufficient condition in Proposition~\ref{prop:partition_tree_cond124}.
Condition 3 is on the \textit{approximation error} side, and we show in
Proposition~\ref{prop:exploration_splits} that it can be satisfied by
introducing exploration splits before the greedy algorithm.

\begin{proposition}[Maximum number of splits for estimation error control]
\label{prop:partition_tree_cond124}
For the best-first split rule on the working box of~\eqref{eq:yn_improved},
with at most $k_N\geq1$ one-coordinate binary splits over a fixed number
$d_Z\geq1$ of coordinates and fixed finite categorical alphabets,
conditions 1, 2, and 4 of Theorem~\ref{the:consistency_1} hold whenever
the split budget satisfies $k_N=o(\sqrt{N/\log N})$. The permitted splits
are threshold tests on continuous coordinates and subset tests on
categorical coordinates.
\end{proposition}

\begin{proposition}[Exploration splits]
\label{prop:exploration_splits}
In the mixed-type setting of Appendix~\ref{app:mixed_type_setup}, assume
i.i.d. observations, $\PXY\ll\nu=\PX\otimes\mu_Y$, fixed dimensions and
finite categorical alphabets, and fixed continuous strictly monotone
transforms $\psi_r$. Use the working boxes of~\eqref{eq:yn_improved}, with
$\PXY(\bZN^c)\to0$ almost surely. Reserve at most $k_{N,e}$ of the $k_N$
splits for exploration, where the deterministic integer budgets satisfy
$1\leq k_{N,e}\leq k_N$ and
\[
k_{N,e}\to\infty,\qquad
\frac{k_{N,e}}{k_N}\to0,\qquad
k_N=o\left(\sqrt{\frac{N}{\log N}}\right).
\]
Let $d_c=\dxc+\dyc$, and let $K_{\mathrm{cat}}$ be the number of combinations
in the fixed categorical alphabets (one if there are no categorical
coordinates). For $d_c>0$ and sufficiently large $N$, choose the largest
integer $r_N\geq0$ such that $K_{\mathrm{cat}}2^{r_Nd_c}-1\leq k_{N,e}$;
set $r_N=0$ when $d_c=0$.

Before greedy growth, resolve every categorical coordinate into singletons
and divide each transformed continuous axis $(0,1)$ into $2^{r_N}$ equal
intervals. Restrict this grid to $\bZN$, using the same half-open boundary
convention at every level. Retain every geometrically nonempty cell,
including those with $n_X=0$, and complete exploration regardless of
empirical count, gain, or minimum-volume thresholds. Subsequent growth may
use any measurable one-coordinate threshold or categorical-subset rule,
subject to the total budget of $k_N$ splits.

The resulting partition $\pi_N$ satisfies condition 3 of
\Cref{the:consistency_1}, and the estimator
in~\eqref{eq:density_formula_improved}, defined as zero whenever its
denominator vanishes, satisfies
\[
\|\fpart-f\|_{L^1(\nu)}\longrightarrow0
\qquad\text{almost surely}.
\]
\end{proposition}

For example, $k_N=\lceil N^{1/3}\rceil$ and
$k_{N,e}=\lceil N^{1/4}\rceil$ satisfy the budget conditions.
Since exploration allows cells with $n_X=0$, Assumption~\ref{ass:admissible_leaves}
need not hold, and we prove consistency directly with the stated zero
convention. Therefore, without additional assumptions on the data distribution,
the consistency guarantee established here concerns the exploration-augmented
construction of Proposition~\ref{prop:exploration_splits}, and it does not
establish shrinkage for the greedy rule used in the experiments. 
For pure greedy tree constructions, structural assumptions on the underlying 
distribution often appear in consistency analyses of other tree methods (see
\citet{chi2022asymptotic,zheng2023consistency}).
Motivated by this line of work, we next give a sufficient gain-margin
condition for the experimental greedy rule: whenever a non-negligible mass
remains in large cells, the split selected by the empirical rule must still
achieve a population gain of a specified minimum size. This condition implies
the shrinkage needed by Theorem~\ref{the:consistency_1}.

\begin{assumption}[Greedy gain margin]\label{ass:greedy_gain_margin}
Let $\pi_{N,t}$ be the partition after $t$ splits of the empirical best-first
rule, and let $s_{N,t}$ be the split it selects next by maximizing empirical
gain. Write $R_P(\pi)$ for population conditional log-loss under the
piecewise-constant conditional density induced by $\pi$, and set
$G_P(s_{N,t}\mid\pi_{N,t})=R_P(\pi_{N,t})-R_P(\pi_{N,t+1})$. For every
$\delta,\epsilon>0$, there
are constants $c_{\delta,\epsilon}>0$ and $N_{\delta,\epsilon}$ such that,
almost surely, for every $N\geq N_{\delta,\epsilon}$ and
$0\leq t<k_N$,
\[
P_{XY}\{z:\diamh(\pi_{N,t}[z])\geq\delta\}\geq\epsilon
\quad\Longrightarrow\quad
G_P(s_{N,t}\mid\pi_{N,t})\geq\frac{c_{\delta,\epsilon}}{t+1}.
\]
Assume also that cumulative population gain along any refinement is bounded
above by a finite constant. For example, this holds when the working domain
has finite $\nu$-measure and the population cell densities are bounded above
and away from zero.
\end{assumption}

\begin{proposition}[Consistency of the empirical greedy rule under a gain margin]
\label{prop:greedy_margin}\label{app:greedy_margin}
Assume the standing hypotheses of Theorem~\ref{the:consistency_1}, conditions
1, 2, and 4 of that theorem (which follow from
Proposition~\ref{prop:partition_tree_cond124} under its split-budget
condition), $k_N\to\infty$, and Assumption~\ref{ass:greedy_gain_margin}.
Then the empirical greedy partition
sequence satisfies condition 3 of Theorem~\ref{the:consistency_1}. In
particular, its conditional density estimator is $L^1(\nu)$-consistent
almost surely.
\end{proposition}

\subsubsection*{Proofs}

\begin{proof}[Proof of Proposition~\ref{prop:partition_tree_cond124}]
Each binary split increases the leaf count by one, so
$m(\A_N)\leq k_N+1$ and $m_X(\A_N)\leq k_N+1$.
For $k_N=o(\sqrt{N/\log N})$, conditions 1 and 4 follow from
\[
\frac{k_N+1}{N}\to0,
\qquad (k_N+1)\sqrt{\frac{\log N}{N}}\to0.
\]
We now bound $\Delta_N^*(\A_N)$ in terms of $k_N$ to verify condition 2.
Let $K_{\max}$ be the largest categorical alphabet size, taking $K_{\max}=0$
if there are no categorical coordinates, and set
$C_N=\max\{2^{K_{\max}},N+1\}$. This bounds the number of distinct sample
splits for one coordinate at one leaf: a threshold has at most $N+1$
positions among the ordered values, and a categorical test has at most
$2^{K_{\max}}$ subsets. The induced sample partition is
determined by (i) membership in the working domain and (ii) the tree splits
inside that domain.

For a fixed working domain $\bZN$, fix arbitrary evaluation points
$\D=\{z_1,\ldots,z_N\}$. Before split $t$, there are at most $t$ leaves,
giving at most $t d_Z C_N$ choices for the next induced sample partition.
For counting, allow steps that leave this partition unchanged by sending
all evaluation points in the chosen leaf to one child and discarding empty
sample intersections. Thus shorter histories can be extended to $k_N$
counting steps without changing the induced partition.
The number of induced partitions is therefore at most
\[
\prod_{t=1}^{k_N}(t d_Z C_N)
=k_N!(d_Z C_N)^{k_N}
\leq(k_Nd_ZC_N)^{k_N}.
\]
The evaluation points $\D$ need not be the training data defining $\bZN$.
As the training data vary, the working domain introduces two endpoints at
each continuous $y$-coordinate, each with at most $N+1$ membership patterns
on $\D$. There are consequently at most $(N+1)^{2\dyc}$ working-domain
membership patterns. Combining each such pattern with the split histories
and maximizing over $\D$ gives
\[
\Delta_N^*(\A_N)\leq(N+1)^{2\dyc}(k_Nd_ZC_N)^{k_N}.
\]
Hence
\begin{align*}
\frac{\log\Delta_N^*(\A_N)}{N}
&\leq\frac{2\dyc\log(N+1)}{N}
+\frac{k_N}{N}\bigl(\log k_N+\log d_Z+\log C_N\bigr)\\
&=O\left(\frac{(k_N+1)\log N}{N}\right)\to0.
\end{align*}
Here $C_N=N+1$ eventually, since the alphabets are fixed, and the split-budget
schedule implies $\log k_N=O(\log N)$ and $k_N\log N/N\to0$.
This proves condition 2.
\end{proof}

\begin{proof}[Proof of Proposition~\ref{prop:exploration_splits}]
\proofparagraph{Shrinkage}
The exploration grid is obtained by categorical subset splits followed by
complete midpoint-refinement rounds. It has at most
$K_{\mathrm{cat}}2^{r_Nd_c}$ cells, so a binary tree constructs it in at most
$K_{\mathrm{cat}}2^{r_Nd_c}-1$ splits. Clipping to $\bZN$ and skipping
geometrically empty children can only reduce this cost. Thus exploration
fits its budget for all sufficiently large $N$. Each leaf then has singleton categories and
transformed coordinate widths at most $2^{-r_N}$; subsequent splits can
only reduce its diameter. For $d_c>0$, maximality of $r_N$ means
$K_{\mathrm{cat}}2^{(r_N+1)d_c}>k_{N,e}+1$. Hence
\[
\sup_{A\in\pi_N}\diamh(A)
\leq d_c2^{-r_N}
\leq2d_c\left(\frac{K_{\mathrm{cat}}}{k_{N,e}+1}\right)^{1/d_c}\to0.
\]
For $d_c=0$, every leaf has diameter zero. Hence condition 3 holds.

\proofparagraph{Estimation error}
Let $f_N^*$ equal the population cell average $\PXY(A)/\nu(A)$ on each
leaf of positive $\nu$-measure, and zero on null cells and outside $\bZN$.
If $\nu(A)>0$ and $\hPX(A_X)>0$, adding and subtracting $\hPXY(A)$ and
using $\hPXY(A)\leq\hPX(A_X)$ gives
\begin{align*}
\int_A|\fpart-f_N^*|\,d\nu
&=\left|\frac{\PX(A_X)}{\hPX(A_X)}\hPXY(A)-\PXY(A)\right|\\
&\leq |\hPXY(A)-\PXY(A)|
+|\PX(A_X)-\hPX(A_X)|.
\end{align*}
If $n_X=0$, then $\hPXY(A)=0$ and the cell's error is exactly $\PXY(A)$,
so the same bound holds. Null cells contribute no integrated error.
Each $A_X$ belongs to $\mathcal C_X$, so summing over at most $k_N+1$
leaves yields
\[
\|\fpart-f_N^*\|_{L^1(\nu)}
\leq\sum_{A\in\pi_N}|\hPXY(A)-\PXY(A)|
+(k_N+1)\sup_{U\in\mathcal C_X}|\hPX(U)-\PX(U)|\to0
\]
almost surely. After adding the outside cell, the counting bounds in
\Cref{prop:partition_tree_cond124} verify the hypotheses of
\Cref{cor:lugosi1}, so the first term vanishes. The supremum is
$O(\sqrt{\log N/N})$ by \Cref{lem:vc_dev_PX_as_only}; the split budget
therefore makes the second term vanish as well. These uniform bounds apply
to the fixed class of covariate boxes and the enclosing family of coordinate
trees with at most $k_N$ splits and box truncation; rational thresholds and
endpoints give a countable pointwise-dense family for measurability.

\proofparagraph{Approximation error}
For any $\epsilon>0$, choose an integrable step function $g$, constant on
a fixed dyadic grid with singleton categories, such that
$\|f-g\|_{L^1(\nu)}<\epsilon$. Such approximations exist because these
nested grids generate the Borel sets, and, across all refinement levels,
their cells with bounded outcome coordinates have finite $\nu$-measure
and cover $\Z$.
Eventually the exploration grid refines this fixed grid, so $g$ is constant
on every leaf. Averaging $f-g$ within each leaf cannot increase its
$L^1$ norm. By the triangle inequality and since $f_N^*=0$ outside $\bZN$,
we obtain
\begin{align*}
\|f_N^*-f\|_{L^1(\nu)}
&\leq\|f_N^*-g\|_{L^1(\nu;\bZN)}
+\|g-f\|_{L^1(\nu;\bZN)}+\PXY(\bZN^c)\\
&\leq2\|f-g\|_{L^1(\nu)}+\PXY(\bZN^c)
\leq2\epsilon+\PXY(\bZN^c).
\end{align*}
Mass capture and arbitrariness of $\epsilon$ give convergence to zero almost surely.
Combining the two errors proves consistency.
\end{proof}

\begin{proof}[Proof of Proposition~\ref{prop:greedy_margin}]
Fix $\delta,\epsilon>0$. The mass of points whose current leaf has
$\diamh\geq\delta$ is nonincreasing under refinement. If this mass were at
least $\epsilon$ after all $k_N$ splits, it would have been at least
$\epsilon$ at every preceding step. The margin condition would force the
cumulative population gain to be at least
\[
c_{\delta,\epsilon}\sum_{t=0}^{k_N-1}\frac{1}{t+1},
\]
which diverges as $N\to\infty$, contradicting its finite upper bound. Hence,
for each fixed $\delta,\epsilon>0$, the terminal mass is eventually below
$\epsilon$. Taking countable sequences $\delta,\epsilon\downarrow0$ gives
condition 3, and~\cref{the:consistency_1} yields the claim.
\end{proof}

\clearpage

\section{Experiments}\label{app:experiments}

In this appendix, we describe the datasets used in the experiments, the hyperparameter tuning search space, and provide results on classification accuracy and Root Mean Squared Error (RMSE) for regression tasks.

\subsection{Datasets}

\Cref{tab:classdataset,tab:regdatasets} summarize the information for the datasets used in this experiment, ordered by sample size. We use datasets ranging from 150 (Iris) to 48842 (Adult) samples in classification. The dataset with the largest alphabet for the target variable is Letter Recognition.

In regression, datasets range from 442 (Diabetes) to 45730 samples. The Naval Propulsion Plant dataset has the largest number of features.

\begin{table}[h]
  \centering
    \caption{Classification datasets ordered by sample size.}\label{tab:classdataset}
    \begin{tabular}{lrrrl}
       \toprule
         \multicolumn{1}{c}{Dataset} & \multicolumn{1}{c}{Size} & \multicolumn{1}{c}{$d_x$} & \multicolumn{1}{c}{Classes} & \multicolumn{1}{c}{Reference}\\
       \midrule
         Iris  & 150 & 4 & 3 & \citet{data:fisher1936use}\\
        Breast Cancer  & 286 & 9 & 2 & \citet{data:breast_cancer_14} \\
        Wine  & 1599 & 11 & 6 &  \citet{data:wine}\\
        Digits  & 1797 & 64 & 10 & \citet{data:digits}\\
        Spam Base & 4601 & 57 & 2 & \citet{data:spambase} \\
        Support2  & 9105 & 42 & 2 & \citet{data:support2}\\
        Letter Recognition  & 20000 & 16 & 26 & \citet{data:letter}\\
        Bank Marketing  & 45211 & 16 & 2 & \citet{data:bank}\\
        Adult  & 48842 & 14 & 4 & \citet{data:adult}\\
 \bottomrule
 \end{tabular} 
\end{table}

\begin{table}[h]
    \centering
    \caption{Regression datasets ordered by sample size.}\label{tab:regdatasets}
    \begin{tabular}{lrrl}
      \toprule
        \multicolumn{1}{c}{Dataset} & \multicolumn{1}{c}{Sample size} & \multicolumn{1}{c}{Features} & \multicolumn{1}{c}{Reference}\\
        \midrule
Diabetes & 442 & 10 & \citet{data:diabetes}\\
Boston & 506 & 13 & \citet{data:boston} \\
Energy Efficiency & 768 & 8 & \citet{data:energy} \\
Concrete Compressive Strength & 1030 & 8 & \citet{data:concrete}\\
Kin8Nm & 8192 & 8 & \citet{data:kin8nm} \\
Air Quality & 9357 & 12 & \citet{data:air}\\
Power Plant & 9568 & 4 & \citet{data:power} \\
Naval Propulsion Plant & 11934 & 14 & \citet{data:naval} \\
California Housing & 20640 & 8 & \citet{data:california}\\
Physicochemical Protein & 45730 & 9 & \citet{data:protein} \\
        \bottomrule
    \end{tabular}    
\end{table}

\subsection{Hyperparameter Tuning}
\label{app:hyperparameter}

Methods with tunable hyperparameters were selected using Optuna with 200 trials and a 20-minute timeout per model.  Cross-validation employed an 80/20 train-test split within each fold. For classification,  the scoring metric was negative log-loss, while for regression, log-loss from the skpro~\cite{gressmann2018probabilistic} library was used to optimize probabilistic predictions.

\subsubsection{Classification Methods}

\textbf{\pt} was tuned over two regularization parameters: the minimum number of 
samples required at a leaf node (\texttt{min\_samples\_leaf} $\in [1, 100]$) and the minimum 
number of samples in the feature space at each leaf (\texttt{min\_samples\_leaf\_x} $\in [1, 400]$).

\textbf{\pf}, the bagging extension of \pt, included additional hyperparameters controlling the subsampling strategy: the fraction of features to consider 
at each split (\texttt{max\_features} $\in [0.7, 1.0]$) and the fraction of samples to draw for each tree (\texttt{max\_samples} $\in [0.7, 1.0]$), along with the same leaf constraints 
as \pt.

\textbf{Random Forest} was tuned over tree depth (\texttt{max\_depth} $\in [3, 50]$), 
minimum samples required to split (\texttt{min\_samples\_split} $\in [2, 150]$), and 
minimum samples at leaf nodes (\texttt{min\_samples\_leaf} $\in [1, 100]$). Objective was set to log-loss.

\textbf{Decision Tree} used the same hyperparameter space as Random Forest, with an 
additional parameter for minimum impurity decrease (\texttt{min\_impurity\_decrease} $\in [0.0, 0.1]$). U

\subsubsection{Regression Methods}

For probabilistic regression, several methods were wrapped with the \texttt{ResidualDouble} 
framework from skpro, which fits a secondary model to predict residual variance, enabling 
uncertainty quantification from point estimators.

\textbf{\pt} (regression) was tuned over four parameters: the minimum samples 
in the target space (\texttt{min\_samples\_leaf\_y} $\in [1, 400]$), minimum samples in 
the feature space (\texttt{min\_samples\_leaf\_x} $\in [1, 400]$), a boundary expansion 
factor (\texttt{boundaries\_expansion\_factor} $\in [0.001, 0.1]$) that controls density  estimation at partition boundaries, and \texttt{min\_target\_volume} $\in [0.01, 0.2]$, that limits the minimum volume $\mu_Y(A_y)$ of a cell to be valid. 

\textbf{\pf} (regression) included the same parameters as \pt, 
plus subsampling controls: \texttt{max\_features} $\in [0.7, 1.0]$ and 
\texttt{max\_samples} $\in [0.7, 1.0]$.

\textbf{Random Forest with ResidualDouble \citep{gressmannskpro}} was tuned over tree depth 
(\texttt{max\_depth} $\in [3, 50]$), minimum samples to split (\texttt{min\_samples\_split} $\in [2, 150]$), 
minimum samples at leaves (\texttt{min\_samples\_leaf} $\in [1, 100]$), minimum impurity 
decrease (\texttt{min\_impurity\_decrease} $\in [0.0, 0.1]$), and maximum features 
(\texttt{max\_features} $\in [0.5, 1.0]$).

\textbf{Decision Tree with ResidualDouble} was tuned over depth, split constraints, 
leaf size, and impurity decrease, using the same ranges as Random Forest.

\textbf{CADET} was tuned over tree depth (\texttt{max\_depth} $\in [3, 50]$), 
minimum samples to split (\texttt{min\_samples\_split} $\in [2, 150]$), and minimum 
samples at leaves (\texttt{min\_samples\_leaf} $\in [1, 100]$).

\textbf{CDTree} was used with its recommended configuration without hyperparameter tuning, as suggested by the original implementation.

For Random Forest, Decision Tree and CADET in both settings, categorical features 
were one-hot encoded via a preprocessing pipeline, while \pt, and 
\pf\ handle categorical features natively.

\subsection{Additional regression comparisons}
\label{app:additional_regression_comparisons}

Per-dataset results for the KCDE and no-tuning comparisons are reported in~\Cref{tab:kcde_comparison,tab:notuned_comparison} as mean and standard deviation of log-loss; the final row gives average ranks across datasets. Both comparisons use the same five cross-validation folds as the main benchmark.

For KCDE, we used the conditional kernel density estimator from \texttt{statsmodels}~\cite{seabold2010statsmodels}, fixed its bandwidth method to \texttt{normal\_reference}, and tuned a multiplicative bandwidth scale in $[0.05,2]$ using the same Optuna budget as the other tuned methods. KCDE has average rank 2.10, compared with 2.70 for \pt\ and 1.20 for \pf.

For the no-tuning comparison with CDTree, \pt\ and \pf\ use one fixed configuration across all ten regression datasets: \texttt{min\_samples\_leaf}, \texttt{min\_samples\_leaf\_x}, and \texttt{min\_samples\_leaf\_y} are set to 1; \texttt{min\_split\_gain} is $10^{-9}$; and \texttt{min\_target\_volume} is 0.05. The boundary expansion factor is 0.1, with maximum depth 1000 and 10,000 iterations. No hyperparameters are selected per dataset. CDTree has average rank 2.10, \pt\ 2.70, and \pf\ 1.20.

\begin{table*}[t]
\centering
\caption{Regression log-loss for KCDE and Partition Tree methods over five-fold cross-validation (mean $\pm$ standard deviation; lower is better). Bold indicates the lowest mean log-loss in each row. The final row reports average rank across datasets.}
\label{tab:kcde_comparison}
\begin{tabular}{lccc}
\toprule
Dataset & KCDE & Partition Tree & Partition Forest \\
\midrule
Diabetes & 5.56 $\pm$ 0.07 & 5.65 $\pm$ 0.11 & \textbf{5.49 $\pm$ 0.10} \\
Boston & 2.63 $\pm$ 0.20 & 2.90 $\pm$ 0.17 & \textbf{2.53 $\pm$ 0.13} \\
Energy & \textbf{1.84 $\pm$ 0.07} & 2.36 $\pm$ 0.18 & 2.06 $\pm$ 0.07 \\
Concrete & 3.58 $\pm$ 0.04 & 3.70 $\pm$ 0.06 & \textbf{3.36 $\pm$ 0.03} \\
Kin8nm & \textbf{-0.57 $\pm$ 0.02} & -0.18 $\pm$ 0.02 & -0.45 $\pm$ 0.02 \\
Air & 3.41 $\pm$ 0.09 & 2.12 $\pm$ 0.11 & \textbf{2.02 $\pm$ 0.10} \\
Power & 2.61 $\pm$ 0.03 & 2.87 $\pm$ 0.02 & \textbf{2.60 $\pm$ 0.02} \\
Naval & -3.79 $\pm$ 0.02 & -4.38 $\pm$ 0.15 & \textbf{-4.72 $\pm$ 0.01} \\
California & 0.79 $\pm$ 0.02 & 0.69 $\pm$ 0.02 & \textbf{0.42 $\pm$ 0.02} \\
Protein & 2.19 $\pm$ 0.02 & 2.22 $\pm$ 0.02 & \textbf{2.10 $\pm$ 0.02} \\
\midrule
Average rank & 2.10 & 2.70 & 1.20 \\
\bottomrule
\end{tabular}
\end{table*}

\begin{table*}[t]
\centering
\caption{Regression log-loss for CDTree and fixed-configuration Partition Tree methods over five-fold cross-validation (mean $\pm$ standard deviation; lower is better). Bold indicates the lowest mean log-loss in each row. The final row reports average rank across datasets.}
\label{tab:notuned_comparison}
\begin{tabular}{lccc}
\toprule
Dataset & CDTree & \makecell[c]{Partition Tree\\(fixed)} & \makecell[c]{Partition Forest\\(fixed)} \\
\midrule
Diabetes & \textbf{5.58 $\pm$ 0.06} & 6.27 $\pm$ 0.14 & 5.82 $\pm$ 0.08 \\
Boston & 2.88 $\pm$ 0.14 & 3.03 $\pm$ 0.07 & \textbf{2.75 $\pm$ 0.07} \\
Energy & \textbf{1.55 $\pm$ 0.16} & 2.65 $\pm$ 0.12 & 2.58 $\pm$ 0.09 \\
Concrete & 3.66 $\pm$ 0.06 & 3.84 $\pm$ 0.10 & \textbf{3.59 $\pm$ 0.08} \\
Kin8nm & -0.18 $\pm$ 0.01 & 0.42 $\pm$ 0.06 & \textbf{-0.18 $\pm$ 0.04} \\
Air & 6.42 $\pm$ 0.00 & 3.08 $\pm$ 0.05 & \textbf{3.03 $\pm$ 0.05} \\
Power & 2.93 $\pm$ 0.03 & 2.92 $\pm$ 0.04 & \textbf{2.91 $\pm$ 0.03} \\
Naval & -4.29 $\pm$ 0.03 & -4.41 $\pm$ 0.03 & \textbf{-4.65 $\pm$ 0.03} \\
California & 0.58 $\pm$ 0.02 & 0.96 $\pm$ 0.06 & \textbf{0.49 $\pm$ 0.03} \\
Protein & 2.18 $\pm$ 0.02 & 2.46 $\pm$ 0.02 & \textbf{2.02 $\pm$ 0.02} \\
\midrule
Average rank & 2.10 & 2.70 & 1.20 \\
\bottomrule
\end{tabular}
\end{table*}

\subsection{Extended classification and regression results}
\label{app:extended_results}

\Cref{tab:ClassifAccuracy} reports the classification accuracy of the evaluated methods, and \Cref{tab:RegRMSESpread} reports regression RMSE. Both tables report mean and standard deviation across five-fold cross-validation.

Across both classification and regression, the results follow a consistent trend. \pt\ outperforms CART for both probabilistic and point-estimation tasks. Among bagging methods, however, performance depends on the metric: \pf\ achieves lower log-loss, whereas Random Forests obtain higher classification accuracy and lower RMSE in regression.

\begin{table*}[h]
\centering

\caption{Accuracy results over 5-fold cross-validation for bagging and single tree methods on classification tasks (mean ± std). Higher is better. Best for bagging and single tree methods in bold. The last row represents the average ranking of each method in its group. \underline{\textit{Underline}} indicates statistical ties according to a paired t-test at the 5\% significance level (run separately within bagging and single tree groups)\label{tb:classif_extended_results}.
}
\begin{tabular}{lcccc}
\toprule
 & \multicolumn{2}{c}{Single tree} & \multicolumn{2}{c}{Bagging} \\
\cmidrule(lr){2-3} \cmidrule(lr){4-5}
 & Partition Tree & CART  & Partition Forest & RF  \\
\midrule
iris & \underline{\textit{0.78 $\pm$ 0.27}} & \textbf{0.81 $\pm$ 0.19} & \underline{\textit{0.83 $\pm$ 0.30}} & \textbf{0.95 $\pm$ 0.02} \\
breast & \underline{\textit{0.69 $\pm$ 0.04}} & \textbf{0.72 $\pm$ 0.06} & \textbf{0.73 $\pm$ 0.04} & \underline{\textit{0.72 $\pm$ 0.07}} \\
wine & \textbf{0.59 $\pm$ 0.03} & 0.55 $\pm$ 0.04 & \underline{\textit{0.67 $\pm$ 0.04}} & \textbf{0.67 $\pm$ 0.04} \\
digits & \textbf{0.87 $\pm$ 0.02} & 0.70 $\pm$ 0.06 & 0.95 $\pm$ 0.01 & \textbf{0.97 $\pm$ 0.00} \\
spam & \textbf{0.93 $\pm$ 0.01} & 0.90 $\pm$ 0.02 & 0.95 $\pm$ 0.01 & \textbf{0.95 $\pm$ 0.01} \\
support2 & \textbf{0.90 $\pm$ 0.01} & \underline{\textit{0.89 $\pm$ 0.01}} & \textbf{0.91 $\pm$ 0.00} & \underline{\textit{0.90 $\pm$ 0.00}} \\
letter & \textbf{0.83 $\pm$ 0.06} & 0.64 $\pm$ 0.02 & 0.89 $\pm$ 0.00 & \textbf{0.96 $\pm$ 0.00} \\
bank & \textbf{0.90 $\pm$ 0.00} & \underline{\textit{0.90 $\pm$ 0.00}} & \textbf{0.91 $\pm$ 0.00} & 0.90 $\pm$ 0.00 \\
adult & \textbf{0.59 $\pm$ 0.01} & \underline{\textit{0.59 $\pm$ 0.01}} & \underline{\textit{0.60 $\pm$ 0.01}} & \textbf{0.60 $\pm$ 0.01} \\
\midrule
Avg. rank & 1.22 & 1.78 & 1.67 & 1.33 \\
\bottomrule
\end{tabular}
\label{tab:ClassifAccuracy}

\end{table*}

\begin{table}[!htpb]
\centering
\caption{RMSE results over 5-fold cross-validation for bagging and single tree methods on regression task (mean ± std). Best for each group in bold (lower is better). The last row represents the average ranking of each method in its group.  \underline{\textit{Underline}} indicates statistical ties according to a paired t-test at the 5\% significance level (run separately within bagging and single tree models)}
\begin{adjustbox}{width=\columnwidth}
\begin{tabular}{lcccccc}
\toprule
 & \multicolumn{4}{c}{Single tree} & \multicolumn{2}{c}{Bagging} \\
\cmidrule(lr){2-5} \cmidrule(lr){6-7}
 & Partition Tree & CART & CADET & CDTree & Partition Forest & RF \\
\midrule
diabetes & \textbf{60.85 $\pm$ 3.70} & 66.07 $\pm$ 4.36 & \underline{\textit{63.37 $\pm$ 2.75}} & \underline{\textit{63.87 $\pm$ 4.32}} & \textbf{56.16 $\pm$ 2.01} & \underline{\textit{56.62 $\pm$ 2.59}} \\
boston & \underline{\textit{4.95 $\pm$ 0.43}} & \textbf{4.56 $\pm$ 1.34} & \underline{\textit{4.57 $\pm$ 0.83}} & \underline{\textit{4.86 $\pm$ 0.86}} & 4.16 $\pm$ 0.70 & \textbf{3.18 $\pm$ 0.66} \\
energy & 4.46 $\pm$ 0.35 & 1.01 $\pm$ 0.09 & \textbf{0.68 $\pm$ 0.04} & 4.88 $\pm$ 0.17 & 4.58 $\pm$ 0.42 & \textbf{0.48 $\pm$ 0.03} \\
concrete & \underline{\textit{8.17 $\pm$ 0.40}} & \underline{\textit{8.34 $\pm$ 0.46}} & \textbf{7.57 $\pm$ 0.55} & 9.01 $\pm$ 0.81 & 7.46 $\pm$ 0.59 & \textbf{5.04 $\pm$ 0.66} \\
kin8nm & \textbf{0.19 $\pm$ 0.00} & 0.20 $\pm$ 0.00 & 0.19 $\pm$ 0.00 & 0.21 $\pm$ 0.00 & 0.16 $\pm$ 0.00 & \textbf{0.14 $\pm$ 0.00} \\
air & 60.40 $\pm$ 4.34 & \textbf{53.89 $\pm$ 2.37} & 56.19 $\pm$ 1.65 & 98.06 $\pm$ 0.09 & 54.66 $\pm$ 5.20 & \textbf{47.83 $\pm$ 2.21} \\
power & \textbf{4.01 $\pm$ 0.16} & 4.44 $\pm$ 0.17 & 4.14 $\pm$ 0.14 & 4.52 $\pm$ 0.13 & 3.68 $\pm$ 0.20 & \textbf{3.32 $\pm$ 0.15} \\
naval & 0.00 $\pm$ 0.00 & 0.01 $\pm$ 0.00 & \textbf{0.00 $\pm$ 0.00} & 0.00 $\pm$ 0.00 & 0.00 $\pm$ 0.00 & \textbf{0.00 $\pm$ 0.00} \\
california & \textbf{0.57 $\pm$ 0.01} & 0.72 $\pm$ 0.01 & 0.61 $\pm$ 0.02 & 0.64 $\pm$ 0.01 & 0.56 $\pm$ 0.01 & \textbf{0.50 $\pm$ 0.01} \\
protein & \textbf{4.39 $\pm$ 0.04} & 5.18 $\pm$ 0.04 & 4.78 $\pm$ 0.12 & 4.63 $\pm$ 0.04 & 4.72 $\pm$ 0.02 & \textbf{3.51 $\pm$ 0.02} \\
\midrule
Avg. rank & 1.90 & 2.90 & 1.80 & 3.40 & 1.90 & 1.10 \\
\bottomrule
\end{tabular}
\end{adjustbox}
\label{tab:RegRMSESpread}
\end{table}

\clearpage

\section{Runtime comparison}\label{app:runtime}

To evaluate the computational efficiency of the proposed \pt\ method, we compared its runtime against those of CADET and CDTree on the Physicochemical Protein dataset. We measured training time across three random sample sizes ($n \in \{100, 1000, 10000\}$), repeating each configuration five times with different random seeds to account for variability.

All models were configured without hyperparameter tuning, using permissive settings (a large maximum depth/iterations and a minimum samples per leaf of 1) to ensure runtime reflects the inherent algorithmic complexity rather than early stopping. 

\Cref{fig:runtime} presents the training time as a function of sample size on a log-log scale. \pt\ consistently achieves the fastest training times across all sample sizes, being approximately $5\times$ faster than CADET and $25\times$ faster than CDTree at $n=10000$. The shaded regions indicate the range between the minimum and maximum observed runtimes across runs.

\begin{figure}[ht]
    \centering
    \includegraphics[width=0.5\columnwidth]{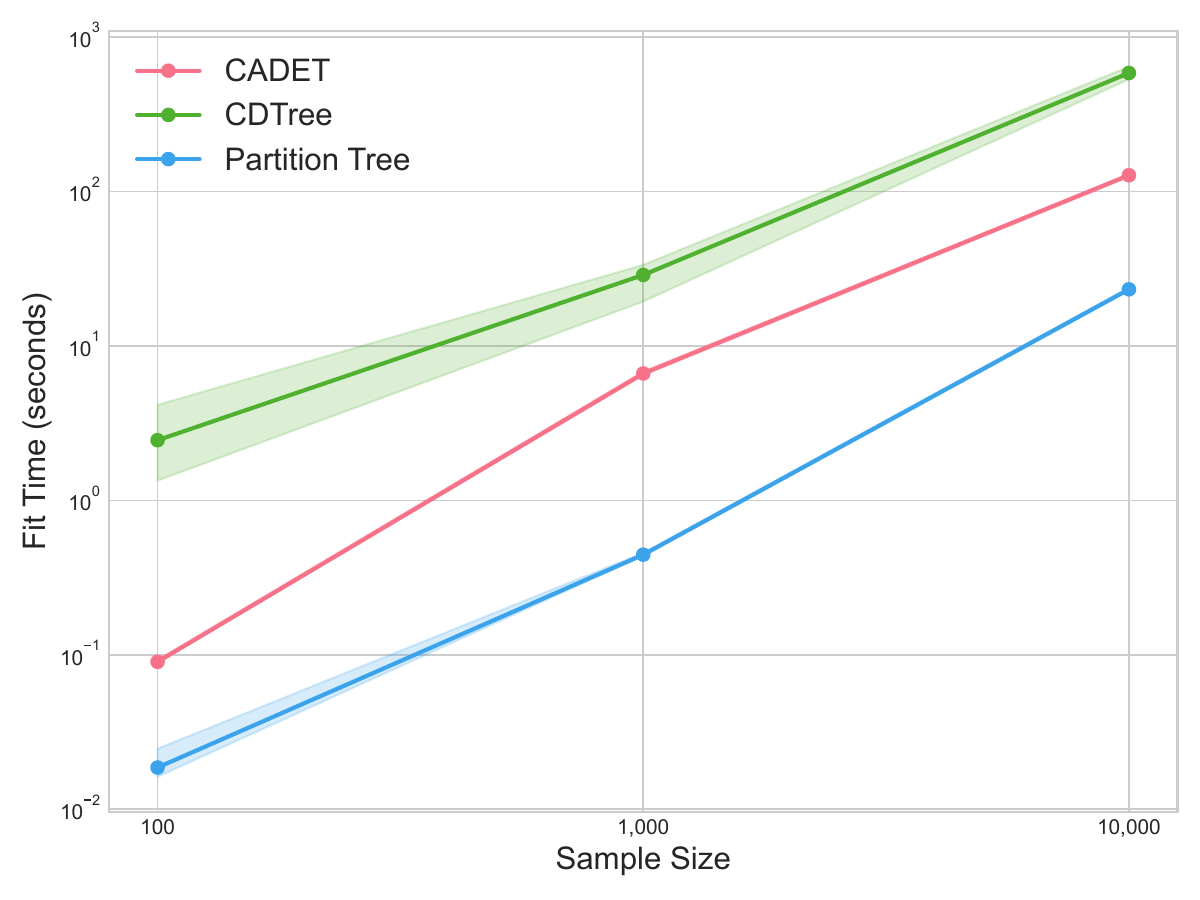}
    \caption{Training time comparison on the Physicochemical Protein dataset across different sample sizes. Solid lines indicate mean runtime; shaded regions show the range across five runs. \pt\ demonstrates consistent computational efficiency and scales favorably compared to both CADET and CDTree.}
    \label{fig:runtime}
\end{figure}

The experiment above isolates worst-case per-fit scaling by disabling hyperparameter tuning and using permissive stopping settings. In the full regression benchmark, however, the end-to-end wall-clock time also depends on the hyperparameter search, so the total model-selection cost can differ from the intrinsic training-time comparison in \cref{fig:runtime}. \Cref{tab:runtime_total_regression} reports the total model time for the nine regression datasets for which all four timings were recorded.

\begin{table}[ht]
\centering
\caption{Total model time (in seconds) for regression tasks, including hyperparameter tuning. Bold indicates the fastest model in each row.}
\label{tab:runtime_total_regression}
\begin{adjustbox}{max width=\columnwidth}
\begin{tabular}{lcccc}
\toprule
Dataset & Partition Tree & CART & CADET & CDTree \\
\midrule
diabetes & 12.83 & \textbf{6.39} & 13.78 & 35.65 \\
boston & 23.53 & \textbf{7.02} & 24.26 & 155.44 \\
energy & 32.00 & \textbf{6.29} & 29.01 & 107.72 \\
concrete & 69.49 & \textbf{7.11} & 41.65 & 173.41 \\
kin8nm & 488.25 & \textbf{13.25} & 466.18 & 843.88 \\
air & 133.77 & 28.63 & 354.37 & \textbf{1.73} \\
power & 387.73 & \textbf{15.74} & 295.87 & 525.33 \\
naval & 849.58 & \textbf{12.99} & 1630.21 & 7842.33 \\
protein & 4946.54 & \textbf{160.55} & 3547.13 & 13755.22 \\
\bottomrule
\end{tabular}
\end{adjustbox}
\end{table}

We also compare inference times on the full test fold for the single-tree regression models. For \pt, we report two prediction implementations: the leaf-scan implementation evaluates each query against the tree's leaf cells, while the tree-traversal implementation follows the tree's split decisions. The timings are in milliseconds per test fold.

\begin{table*}[t]
\centering
\small
\caption{Inference time in milliseconds per full test fold for the single-tree regression methods. Lower is faster; average ranks are computed across datasets.}
\label{tab:inference_time}
\begin{adjustbox}{max width=\textwidth}
\begin{tabular}{lccccc}
\toprule
Dataset & \pt\ (leaf scan) & \pt\ (tree traversal) & CART & CADET & CDTree \\
\midrule
diabetes & 2.3 & 1.7 & \textbf{1.1} & 3.6 & 10.5 \\
boston & 3.3 & 2.9 & \textbf{1.6} & 4.3 & 12.0 \\
energy & 4.8 & 9.3 & \textbf{1.1} & 5.3 & 26.5 \\
concrete & 6.6 & 3.4 & \textbf{1.1} & 6.8 & 18.7 \\
kin8nm & 36.1 & 22.9 & \textbf{1.2} & 42.9 & 137.9 \\
air & 28.8 & 46.2 & \textbf{2.2} & 51.1 & 106.0 \\
power & 77.8 & 43.5 & \textbf{1.2} & 56.4 & 163.4 \\
naval & 377.6 & 67.8 & \textbf{1.5} & 73.6 & 287.7 \\
california & 212.3 & 66.3 & \textbf{1.3} & 117.3 & 452.2 \\
protein & 1125.0 & 281.2 & \textbf{1.6} & 269.7 & 847.8 \\
\midrule
Avg. rank & 3.40 & 2.40 & 1.00 & 3.40 & 4.80 \\
\bottomrule
\end{tabular}
\end{adjustbox}
\end{table*}

\clearpage
\section{Gain-based Feature Importance}
\label{app:feature_importance}

\pt\ retains the interpretability advantages of classical decision trees: each internal node applies a single-coordinate test, and the learned model is a piecewise-constant conditional density on the induced leaves.

To summarize which \emph{covariates} drive the conditional density estimate, we use a gain-based feature importance that is aligned with the training objective.
Recall that every accepted split is chosen to maximize the empirical log-loss gain $\widehat G$ (\Cref{eq:empirical_gain}). For a fitted tree $T$ and a covariate coordinate $j\in\{1,\dots,d_x\}$, define the unnormalized importance
\[
\widetilde I_j(T)
:= \sum_{s\in\mathrm{Int}(T)} \widehat G_s \;\mathbf 1\{\,\text{split }s\text{ uses }X_j\,\},
\]
where $\mathrm{Int}(T)$ is the set of internal nodes and $\widehat G_s$ is the
gain realized at node $s$. We then report the normalized importance
\[
I_j(T) := \frac{\widetilde I_j(T)}{\sum_{s\in\mathrm{Int}(T)} \widehat G_s},
\qquad \sum_{j=1}^{d_x} I_j(T)=1.
\]
This definition matches the standard ``impurity decrease'' importances used for
CART, with $\widehat G$ replacing the classification or regression impurity.
Because \pt\ also allows $Y$-splits to refine the outcome partition,
we restrict $I_j$ to \emph{$X$-splits only} so that importances reflect how the
model adapts to the covariates (rather than how it discretizes the outcome).

\subsection{Comparison with CADET importances}
CADET selects splits using a cross-entropy criterion, so its gain has the same
log-loss structure as ours. Empirically, the resulting importances are closely
aligned across regression benchmarks: \cref{fig:feature_importances_all}
plots CADET importances against \pt\ importances (each point is one
feature in one dataset), and most points concentrate near the equality line,
indicating broad agreement in which covariates matter and at what scale.%
\;\;{\footnotesize\;\;\;\;\;\;\;{\color{white}.}}%
\label{app:featimp_cadet_agreement}

\begin{figure}[tbh]
    \centering
    \includegraphics[width=0.85\linewidth]{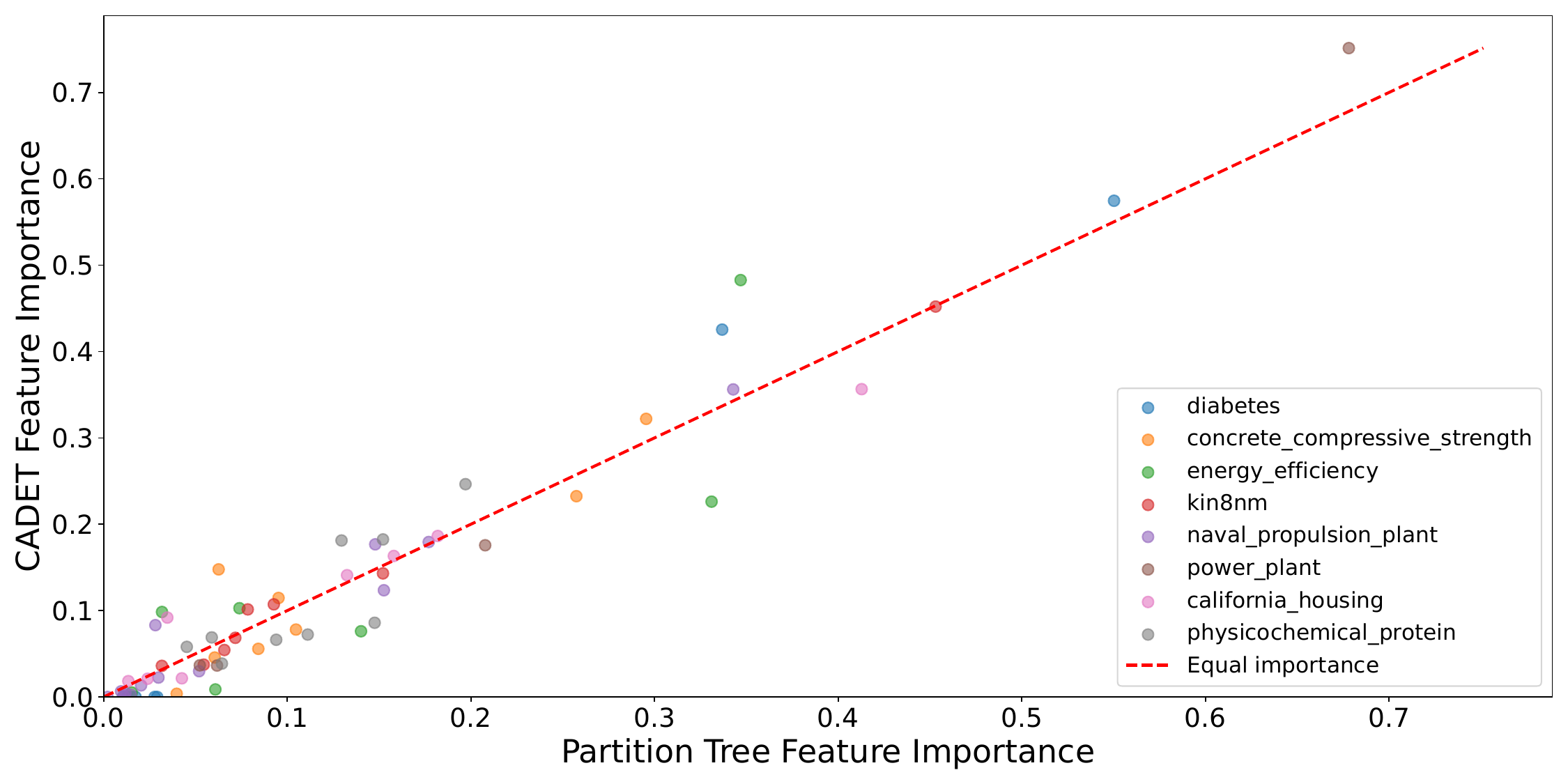}
    \caption{Normalized gain-based feature importances for \pt\ vs.\ CADET across regression datasets for the first fold of the cross-validation. Each marker corresponds to a single feature on a given dataset; the dashed line indicates equal importance. The Air Quality dataset is omitted because CADET returned NaN feature importances.}
    \label{fig:feature_importances_all}
\end{figure}

\subsection{Example: California Housing}
As a concrete illustration, \cref{fig:feature_importance_california} shows the normalized features importance on California Housing for the \pt\ and
CADET. Both methods identify \texttt{MedInc} as the dominant predictor, followed by geographic covariates (\texttt{Longitude}, \texttt{Latitude}) and occupancy statistics, with the remaining variables contributing less. This qualitative agreement is representative of the trends observed across datasets.%
\;\;{\footnotesize\;\;\;\;\;\;\;{\color{white}.}}%


\begin{figure}[!htb]
  \centering
  \begin{subfigure}[c]{.48\linewidth}
    \includegraphics[width=\linewidth]{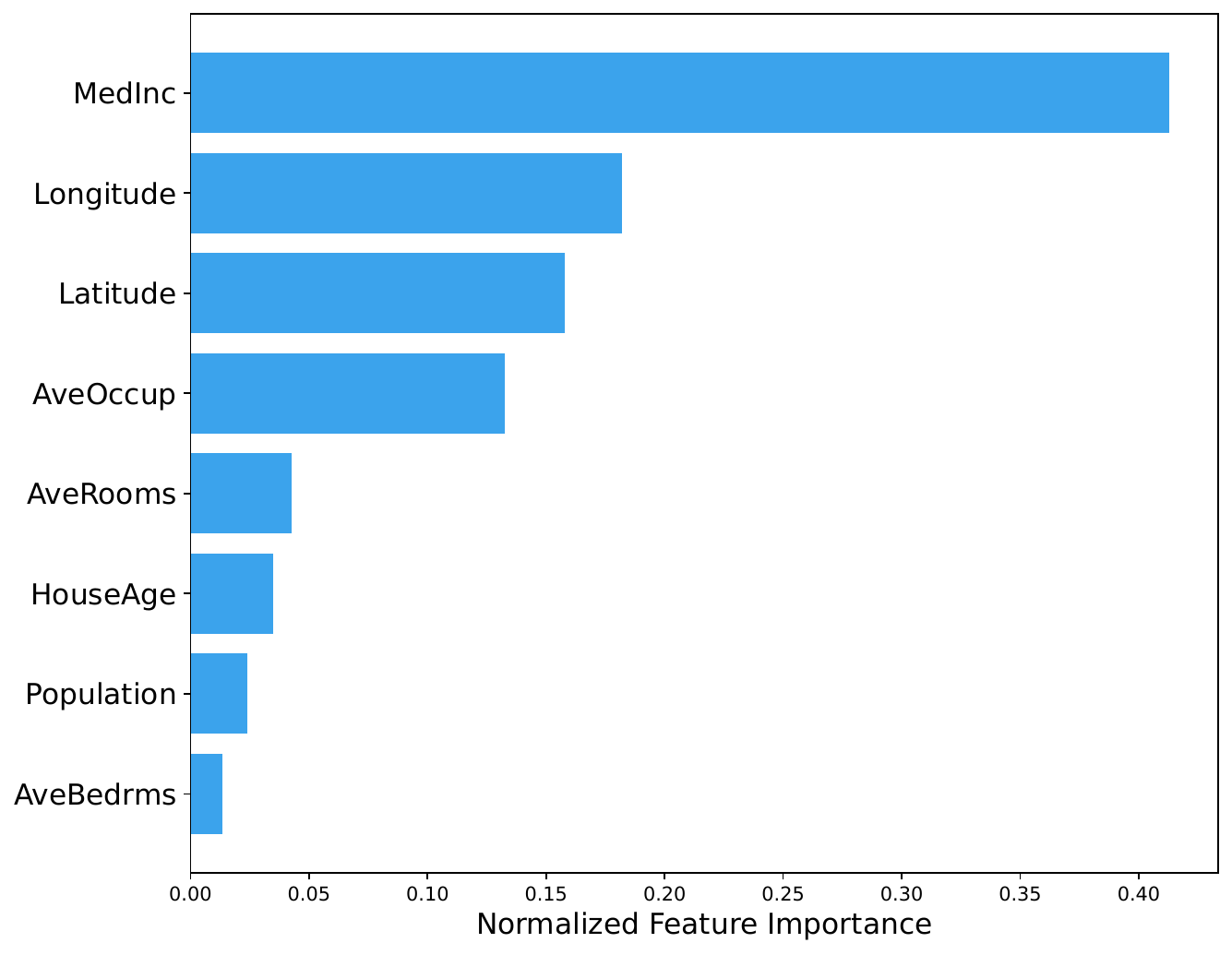}%
    \caption{\pt}
  \end{subfigure}
  \hfill
  \begin{subfigure}[c]{.48\linewidth}
    \includegraphics[width=\linewidth]{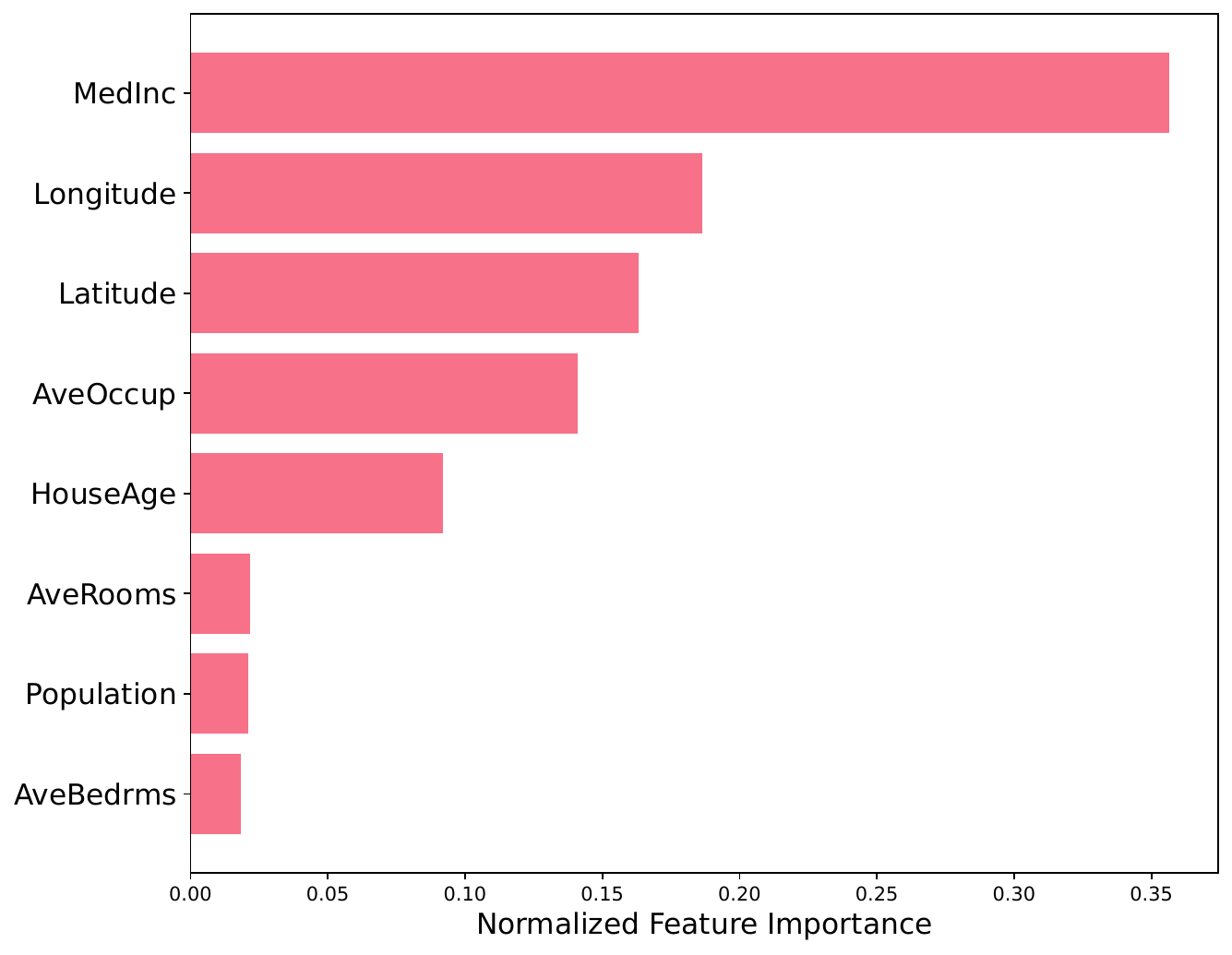}%
    \caption{CADET}
  \end{subfigure}
    \caption{Normalized gain-based feature importance for the California Housing dataset. Both \pt\ and CADET place \texttt{MedInc} as most important, followed by geographic covariates and occupancy-related features.}
  \label{fig:feature_importance_california}
\end{figure}

\paragraph{Caveat.}
As with standard impurity-based importance, gain-based scores can be affected by correlated predictors and by the set of admissible split tests. We therefore use them as a descriptive diagnostic of how the learned partition adapts to $X$, rather than as a causal attribution.

\end{document}